\documentclass{article}

\usepackage{microtype}
\usepackage{graphicx}
\usepackage{multirow}
\usepackage{subfigure}
\usepackage[T1]{fontenc}
\usepackage[utf8]{inputenc}
\usepackage{thmtools,thm-restate}
\usepackage{booktabs} 

\usepackage{hyperref}

\usepackage[accepted]{icml2024}

\usepackage{amsmath}
\usepackage{amssymb}
\usepackage{enumitem}
\usepackage{mathtools}
\usepackage{pythonhighlight}
\usepackage{amsthm}

\usepackage[capitalize,noabbrev]{cleveref}

\theoremstyle{plain}
\newtheorem{theorem}{Theorem}[section]

\newtheorem{lemma}[theorem]{Lemma}

\theoremstyle{definition}
\newtheorem{definition}[theorem]{Definition}

\theoremstyle{remark}
\usepackage{xcolor}

\usepackage[textsize=tiny]{todonotes}

\newcommand{\methodname}{Token-level Direct Preference Optimization}
\newcommand{\methodabb}{TDPO}

\icmltitlerunning{\methodname{}}

\begin{document}

\twocolumn[
\icmltitle{\methodname{}}



\icmlsetsymbol{equal}{*}

\begin{icmlauthorlist}
\icmlauthor{Yongcheng Zeng}{CASIA,CAS}
\icmlauthor{Guoqing Liu}{msra}
\icmlauthor{Weiyu Ma}{CASIA,CAS}
\icmlauthor{Ning Yang}{CASIA}
\icmlauthor{Haifeng Zhang}{CASIA}
\icmlauthor{Jun Wang}{UCL}
\end{icmlauthorlist}

\icmlaffiliation{CASIA}{Institute of Automation, Chinese Academy of Sciences}
\icmlaffiliation{CAS}{School of Artificial Intelligence, University of Chinese Academy of Sciences}
\icmlaffiliation{msra}{Microsoft Research AI4Science}
\icmlaffiliation{UCL}{University College London}

\icmlcorrespondingauthor{Jun Wang}{jun.wang@cs.ucl.ac.uk}

\icmlcorrespondingauthor{Haifeng Zhang}{haifeng.zhang@ia.ac.cn}

\icmlkeywords{Machine Learning, ICML}

\vskip 0.3in
]



\printAffiliationsAndNotice{}  

\begin{abstract}
Fine-tuning pre-trained Large Language Models (LLMs) is essential to align them with human values and intentions. This process often utilizes methods like pairwise comparisons and KL divergence against a reference LLM, focusing on the evaluation of full answers generated by the models. However, the generation of these responses occurs in a token level, following a sequential, auto-regressive fashion. In this paper, we introduce \methodname{} (\methodabb{}), a novel approach to align LLMs with human preferences by optimizing policy at the token level. Unlike previous methods, which face challenges in divergence efficiency, \methodabb{} incorporates forward KL divergence constraints for each token, improving alignment and diversity. Utilizing the Bradley-Terry model for a token-based reward system, TDPO enhances the regulation of KL divergence, while preserving simplicity without the need for explicit reward modeling. Experimental results across various text tasks demonstrate \methodabb{}'s superior performance in balancing alignment with generation diversity.
Notably, fine-tuning with \methodabb{} strikes a better balance than DPO in the controlled sentiment generation and single-turn dialogue datasets, and significantly improves the quality of generated responses compared to both DPO and PPO-based RLHF methods. Our code is open-sourced at \href{https://github.com/Vance0124/Token-level-Direct-Preference-Optimization}{https://github.com/Vance0124/Token-level-Direct-Preference-Optimization}.
\end{abstract}

\section{Introduction}
Large language models (LLMs) \cite{achiam2023gpt, bubeck2023sparks} have demonstrated significant generalization capabilities in various domains including text summarization \cite{stiennon2022learning, Koh_2022}, coding writing \cite{chen2021evaluating, gao2023pal}, and even following human instructions \cite{chung2022scaling, ouyang2022training}.
In order to align LLMs with human intentions, Reinforcement Learning from Human Feedback (RLHF) \cite{christiano2017deep, ouyang2022training, dong2023raft, yuan2023rrhf, liu2023statistical} has emerged as a highly effective method, embodying both stylistic and ethical values \cite{bai2022training, ganguli2022red}. These approaches typically involve the training of a reward model followed by the fine-tuning of the policy model using reinforcement learning (RL).

\begin{figure*}[ht]
\vskip 0.2in
\centering
\subfigure[\label{demo_subfig1}]{\includegraphics[width=0.32\textwidth]{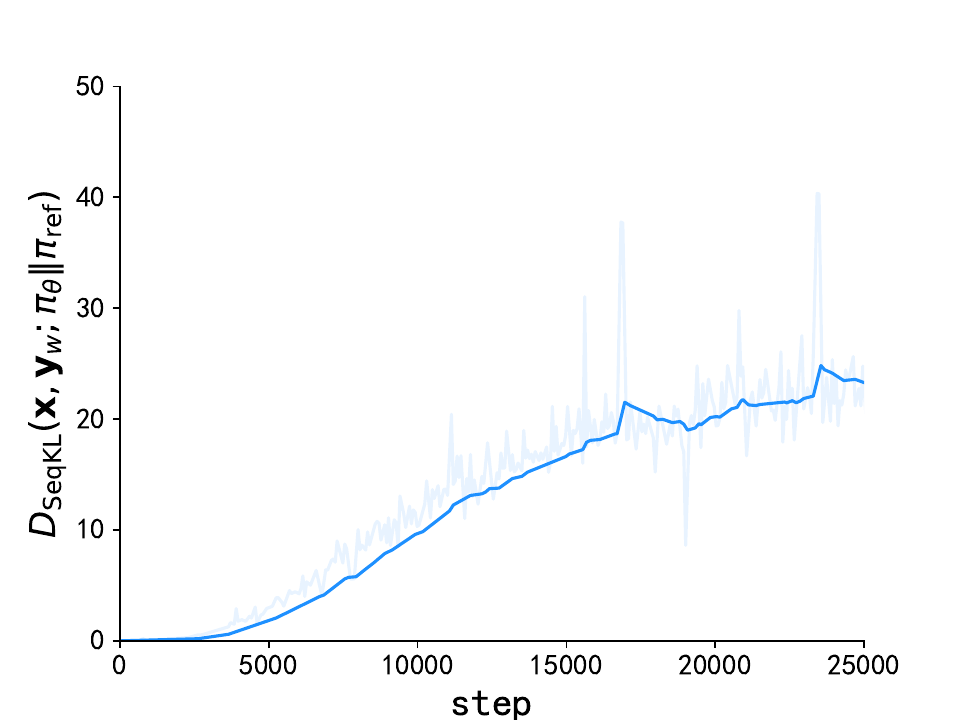}}
\subfigure[\label{demo_subfig2}]{\includegraphics[width=0.32\textwidth]{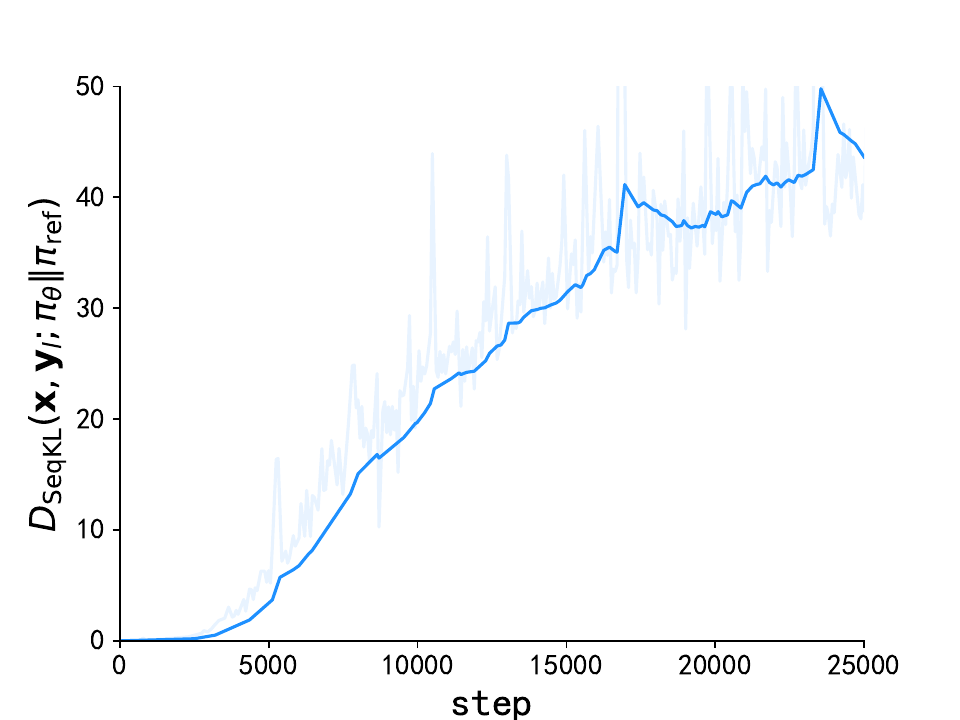}}
\subfigure[\label{demo_subfig3}]{\includegraphics[width=0.32\textwidth]{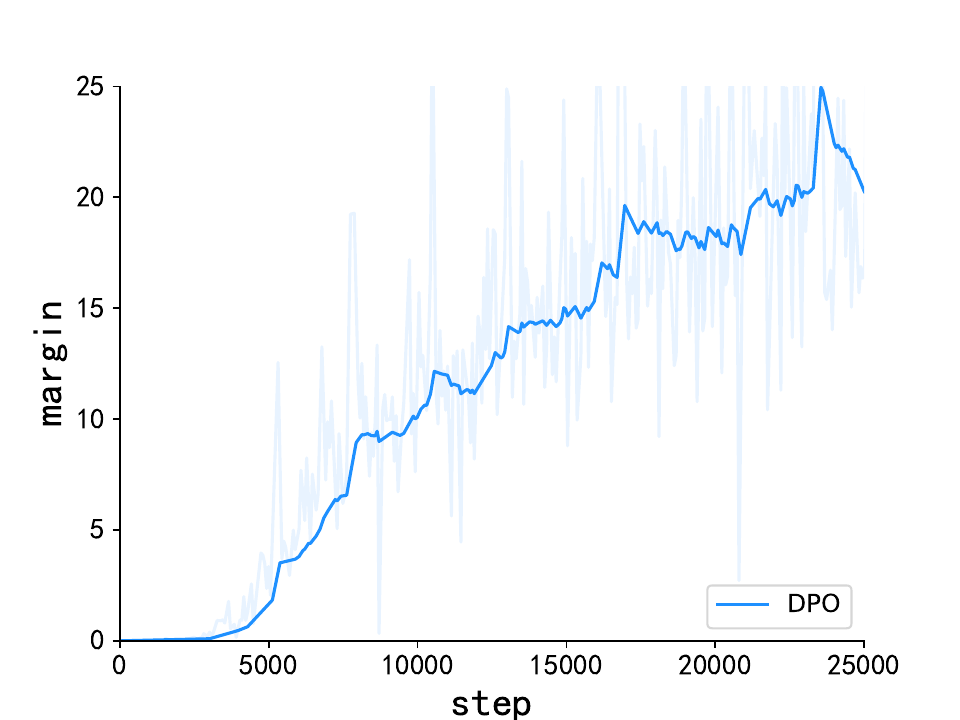}}
  \caption{Sequential KL (SeqKL) divergence of both preferred response and dispreferred responses on IMDb dataset. \cref{demo_subfig1} shows the progression of SeqKL divergence on the preferred responses over training steps. \cref{demo_subfig2} depicts the evolution of SeqKL divergence on the dispreferred responses over the training steps. \cref{demo_subfig3} illustrates the difference between the SeqKL divergence of the dispreferred responses and that of the preferred responses during the training process, namely $margin=|D_{\mathrm{SeqKL}}({x}, {y}_w;\pi_{\mathrm{ref}}\|\pi_{\theta}) - D_{\mathrm{SeqKL}}({x}, {y}_l;\pi_{\mathrm{ref}}\|\pi_{\theta})|$. The definition of SeqKL divergence refers to \cref{def1}.}
  \label{fig1}
\vskip -0.1in
\end{figure*}

Direct Preference Optimization (DPO) \cite{rafailov2023direct} introduces a straightforward and effective technique for training LLMs using pairwise comparisons, without the need for explicitly establishing a reward model.
DPO utilizes KL divergence to ensure that the training process remains closely aligned with a reference Large Language Model (LLM), preventing significant deviations.
In DPO, KL divergence is assessed at the sentence level, reflecting the fact that evaluations are based on complete responses (answers), typically comprising several sentences. However, the generation of these responses occurs sequentially, following an auto-regressive approach. A potential benefit is to examine divergence in relation to a reference LLM on a more granular, token-by-token basis. One approach involves using sequential KL divergence (as defined in \cref{def1}), which monitors the trajectory of the generated responses.  As illustrated in \cref{fig1}, DPO demonstrates a significantly faster increase in KL divergence within the subset of less preferred responses when compared to the subset that is preferred. 
This results in an expanding gap between the two subsets and also indicates that DPO does not effectively control the KL divergence of the dispreferred response subset. This impacts the model's divergence efficiency and ultimately affects its linguistic capabilities and generative diversity. 
Such a limitation highlights the decreased effectiveness of employing KL divergence within the DPO framework, suggesting an area for improvement in its methodology.



The imbalance in the growth rates of the sequential KL divergence is potentially related to the reverse KL divergence constraint employed by DPO. 
The mode-seeking property of reverse KL divergence tends to induce diversity reduction during generation, limiting the model's potential to produce diverse and effective responses \cite{wiher2022decoding, khalifa2020distributional, glaese2022improving, perez2202red}. Built upon DPO, the f-DPO method \cite{wang2023beyond} studies the trade-off between alignment performance and generation diversity of LLMs under different divergence constraints. It highlights the advantages of the mass-covering behavior of forward KL divergence in enhancing model diversity and explores the impact of different divergence constraints. 
Nevertheless, 
f-DPO only independently discusses the changes in model behavior under either the reverse KL divergence or the forward KL divergence constraints. Essentially, it does not fundamentally enhance the DPO algorithm itself but rather strikes a balance between alignment performance and generating diversity by simply swapping different KL divergence constraints.

Inspired by the aforementioned observations, we define and examine the problem of aligning with human preferences from a sequential and token-level standpoint. Some concurrent work has also been conducted in this direction \cite{rafailov2024r,zhong2024dpo}.  We introduce a new method, referred to as \methodname{} (\methodabb{}), which aims to strike a better balance between alignment performance and generation diversity by controlling the KL divergence for each token.
In order to achieve this, we redefine the objective of maximising restricted rewards in a sequential manner. The connection between sentence-level reward and token-level generation is established by the use of the Bellman equation. Afterwards, the Bradley-Terry model \cite{bradley1952rank} is converted into a representation at the token level, demonstrating its close relationship with the Regret Preference Model  \cite{knox2022models, knox2023learning}. By utilizing this method, we effectively integrate forward KL divergence restrictions for each token in the final objective function, resulting in improved regulation of KL divergence.


\methodabb{} maintains the simplicity of DPO while offering improved regulation of KL divergence for aligning LLMs with human preferences. Echoing the strategy of DPO, our method directly optimizes the policy without necessitating explicit reward model learning or policy sampling throughout the training phase. Our experimental results demonstrate the effectiveness of \methodabb{} across multiple text tasks, and gain a notable enhancement in the quality of generated responses in comparison to both DPO and PPO-based RLHF methods. In conclusion,
\methodabb{} stands out for its ability to not only effectively address the issue of excessive KL divergence but also greatly improve divergence efficiency.

\section{Related Works}
The emergence of ChatGPT has catalyzed significant advancements in the field of Large Language Models (LLMs), such as OpenAI's GPT-4 \cite{achiam2023gpt}, Mistral \cite{jiang2023mistral}, and Google's Gemini \cite{team2023gemini}. 
Generally, the training of LLMs involves three stages: initial unsupervised pre-training on massive text corpora to grasp linguistic structures \cite{raffel2020exploring, brown2020language, workshop2022bloom, touvron2023llama}, followed by supervised fine-tuning with task-specific datasets to enhance the LLMs' probability of producing desired responses \cite{taori2023alpaca, chiang2023vicuna, vu2023koala}. However, due to the typically limited and expensive availability of labeled datasets during the supervised fine-tuning stage, the model may retain biases and inaccuracies, manifesting as societal biases \cite{sheng2021societal}, ethical concerns \cite{weidinger2021ethical}, toxicity \cite{rauh2022characteristics}, and hallucinations \cite{huang2023survey}, which necessitates a subsequent AI alignment phase. Noteworthy models achieving significant alignment, such as Zephyr \cite{tunstall2023zephyr} and GPT-4 \cite{achiam2023gpt}, have demonstrated the effectiveness of techniques like Reinforcement Learning from Human Feedback (RLHF) and Direct Preference Optimization (DPO) algorithms.

Reinforcement Learning from Human Feedback (RLHF) has emerged as a cornerstone in aligning LLMs with human values, providing a mechanism to refine model outputs based on qualitative feedback \cite{christiano2017deep, ouyang2022training, bai2022training, song2023preference, touvron2023llama}. This approach has shown considerable promise in making models more responsive to human expectations and ethical considerations by iteratively improving their performance through human-generated feedback. However, the complexity of implementing RLHF, compounded by the inaccuracies in human-generated reward models \cite{wu2023fine}, has prompted the exploration of alternative strategies. Methods like Reward Ranked FineTuning (RAFT) \cite{dong2023raft} and Rank Responses to align Human Feedback (RRHF) \cite{yuan2023rrhf} offer streamlined approaches to alignment, circumventing some of RLHF's inherent challenges. Particularly, Direct Preference Optimization (DPO) \cite{rafailov2023direct} represents a breakthrough in direct policy optimization, addressing the intricacies of balancing model behavior through a nuanced approach to reward function optimization. Nevertheless, the challenge of maintaining linguistic diversity while aligning with human preferences remains a pivotal concern, prompting our proposed \methodname{} (\methodabb{}), which seeks to harmonize the dual objectives of alignment accuracy and expressive range in model outputs.

\section{Preliminaries}
 \label{DPO} 
For language generation, a language model (LM) is prompted with prompt (question) ${x}$ to generate a response (answer) ${y}$, where both ${x}$ and ${y}$ consist of a sequence of tokens. Direct Preference Optimization (DPO) \cite{rafailov2023direct}
 commences with the RL objective from the RLHF:
\begin{equation}
\begin{split}
\max_{\pi_\theta}\ \mathbb{E}_{{x}\sim\mathcal{D},{y}\sim\pi_\theta(\cdot|{x})}&\big[r({x},{y})\\
-\beta D_{\mathrm{KL}}&\big(\pi_\theta(\cdot\mid {x})\big\|\pi_{\mathrm{ref}}(\cdot\mid {x})\big)\big],
\end{split}\label{rlhfob}
\end{equation}
where $\mathcal{D}$ represents the human preference dataset, $r({x},{y})$ denotes the reward function, $\pi_{\mathrm{ref}}(\cdot|{x})$ serves as a reference model, typically chosen the language model after supervised fine-tuning, $\pi_{\theta}$ represents the model undergoing RL fine-tuning, initialized with $\pi_{\theta}=\pi_{\mathrm{ref}}$, and $\beta$ is the coefficient for the reverse KL divergence penalty.

By directly deriving from Eq.~\ref{rlhfob},
DPO establishes a mapping between the reward model and the optimal policy under the reverse KL divergence, obtaining a representation of the reward function concerning the policy:
\begin{equation}
r({x}, {y})=\beta\log\frac{\pi_{\theta}({y}|{x})}{\pi_{\mathrm{ref}}({y}|{x})}+\beta\log Z({x}).
\label{dpo_mapping}
\end{equation}
Here, $Z({x})$ is the partition function. 

To align with human preference, DPO uses the Bradley-Terry model for pairwise comparisons:
\begin{align}
    P_{\mathrm{BT}}({y}_1\succ {y}_2 | {x})&=\frac{\exp(r({x}, {y}_1))}{\exp(r({x}, {y}_1))+\exp(r({x}, {y}_2))}.
    \label{BT_model}
\end{align}

By substituting Eq.~\ref{dpo_mapping} into Eq.~\ref{BT_model} 
and leveraging the negative log-likelihood loss, DPO derives the objective function:
\begin{align}
    u({x}, {y}_w, {y}_l)=\beta\log\frac{\pi_\theta({y}_w\mid {x})}{\pi_{\mathrm{ref}}({y}_w\mid {x})}-\beta\log\frac{\pi_\theta({y}_l\mid {x})}{\pi_{\mathrm{ref}}({y}_l\mid {x})},\nonumber\\
    \mathcal{L}_{\mathrm{DPO}}(\pi_\theta;\pi_{\mathrm{ref}})=-\mathbb{E}_{({x},{y}_w,{y}_l)\sim\mathcal{D}}\left[\log\sigma\left(u({x}, {y}_w, {y}_l)\right)\right],
\end{align}
and the derivative is given as follows:
\begin{equation}
\nabla_\theta\mathcal{L}_{\mathrm{DPO}}(\pi_\theta;\pi_{\mathrm{ref}})=-\mathbb{E}_{({x},{y}_w,{y}_l)\sim\mathcal{D}}
    \left[\sigma\left(-u\right)
    \nabla_\theta u\right],
\end{equation}
where $u$ is the abbreviation of $u({x}, {y}_w, {y}_l)$, $y_w$ and $y_l$ denotes the preferred and dispreferred completion.

\section{Methodology}

In this section, we initially reformulate the constrained reward maximization problem into a token-level form. 
From this, we derive the mapping between the state-action function and the optimal policy. Subsequently, we convert the Bradley-Terry model into token-level representation, establishing its equivalence with the Regret Preference Model. By substituting the mapping relationship into the reward model in token-level format, we obtain the optimization objective solely related to the policy. Finally, we conduct a formalized analysis of this optimization objective in terms of derivatives and, based on this, derive the ultimate loss function for \methodabb{}.


\subsection{Markov Decision Process under Token Rewards} \label{Nota}  
To model the sequential, auto-regressive generation, we extend the sentence-level formulation in Section~\ref{DPO} by considering that the response consists of $T$ tokens ${y} = y^{<T+1}\coloneqq [y^1, y^2, ..., y^T]$, where $y^t\in \mathcal{Y}$, and $\mathcal{Y}$ represents the alphabet (vocabulary). Additionally, we assume $y^{<1}=[\ ]$.
 Given a prompt ${x}$ and the first $t-1$ tokens $y^{<t}$ of the response ${y}$, the LM predicts the probability distribution of the next token $\pi_{\theta}(\cdot|[{x}, y^{<t}])$. 

When modeling text generation as a Markov decision process \cite{puterman2014markov}, a state is a combination of the prompt and the generated response up to the current step, denoted as $s_t = [{x}, y^{<t}]$. An action corresponds to the next generated token, denoted as $a_t = y^t$, and the token-wise reward is defined as $R_t := R(s_t, a_t) = R([{x}, y^{<t}], y^t)$. 

Expanding on the provided definitions, we establish the state-action function $Q_{\pi}$, the state value function $V_{\pi}$ and the advantage function $A_{\pi}$ for a policy $\pi$:
\begin{equation}
    {\small
    \begin{aligned}
    Q_{\pi}([{x},y^{<t}],y^t) & = \mathbb{E}_{\pi}\left[\sum_{k=0}^{\infty}\gamma^{k} R_{t+k}\bigg|s_{t}=[{x},y^{<t}], a_t = y^t\right],\\
    V_{\pi}([{x},y^{<t}]) & = \mathbb{E}_{\pi}\big[Q_{\pi}([{x},y^{<t}],y^t)\big|s_t=[{x},y^{<t}]\big],\\
    A_{\pi}([{x},y^{<t}], y^t) & = Q_{\pi}([{x},y^{<t}],y^t)-V_{\pi}([{x},y^{<t}]).
    \end{aligned}}\label{QVA}
\end{equation}
where $\gamma$ represents the discount factor. In this paper, we set $\gamma=1$.

\subsection{Token-Level Optimization}
DPO's objective function in Eq.~\ref{rlhfob} operates at the sentence level. In contrast, we propose an alternative token-level objective function:
\begin{equation}
    \begin{split}
    \max_{\pi_{\theta}} \ &\mathbb{E}_{{x}, {y}^{<t}\sim\mathcal{D},z\sim \pi_{\theta}(\cdot|[{x},y^{<t}])}\big[A_{\pi_{\mathrm{
ref}}}([{x},y^{<t}], z)\\
    &-\beta D_{\mathrm{KL}}\left(\pi_{\theta}(\cdot|[{x},y^{<t}])||\pi_{\mathrm{ref}}(\cdot|[{x},y^{<t}])\right)\big].
    \end{split}\label{rpo_obj}
\end{equation}

The objective function is inspired by Trust Region Policy Optimization
 (TRPO) \cite{schulman2015trust}. As demonstrated in Lemma \ref{lemma_2}, maximizing the objective function in Eq.~\ref{rpo_obj} will result in policy improvements in terms of expected return.
 \begin{restatable}{lemma}{mylemmatwo}\label{lemma_2}
    Given two policies $\pi$ and $\Tilde{\pi}$, if for any state $s_t=[{x}, y^{<t}]$, $\mathbb{E}_{z\sim \Tilde{\pi}
    }\left[A_{\pi}([{x}, y^{<t}], z )\right] \ge 0$, then we can conclude:
\begin{equation*}
    \mathbb{E}_{{x}\sim \mathcal{D}}\left[V_{\Tilde{\pi}}([{x}])\right] \ge \mathbb{E}_{{x}\sim \mathcal{D}}\left[V_{\pi}([{x}])\right],
\end{equation*}
\end{restatable}
The proof is provided in Appendix \ref{A_2}.

Notably, to maintain generation diversity and prevent the model from hacking some high-reward answers,
we incorporate reverse KL divergence for each token in our token-level objective function, which prevents the model from deviating too far from the reference model distribution.

Starting from the token-level objective function in Eq.~\ref{rpo_obj}, we can directly derive the mapping between the state-action function $Q_{\pi}$ and the optimal policy $\pi_{\theta}^*$. We summarize this relationship in the following lemma.
\begin{restatable}{lemma}{mylemmafour}\label{lemma4}
    The constrained problem in Eq.~\ref{rpo_obj} has the closed-form solution:
\begin{equation}
\begin{aligned}
    \pi_{\theta}^*(z|[{x}&,y^{<t}])= \\
     &\frac{\pi_{\mathrm{ref}}(z|[{x},y^{<t}])\exp\left(\frac{1}{\beta}Q_{\pi_{\mathrm{ref}}}([{x},y^{<t}],z)\right)}{Z([{x},y^{<t}];\beta)},
\end{aligned}\label{eq_3}
\end{equation}
where $Z([{x},y^{<t}];\beta) = \mathbb{E}_{z\sim \pi_{\mathrm{ref}}(\cdot|[{x},y^{<t}])}e^{\frac{1}{\beta}Q_{\pi_{\mathrm{ref}}}([{x},y^{<t}],z)}$
is the partition function.
\end{restatable}
 See \cref{A_4} for more details. 
 
 To obtain the optimal policy $\pi_{\theta}^*$ from Eq.~\ref{eq_3}, we must estimate the state-action function $Q_{\pi_{\mathrm{ref}}}$ and the partition function $Z(\cdot)$. However, ensuring the accuracy of the state-action function $Q_{\pi}$ at each state and action is challenging, and estimating the partition function $Z(\cdot)$ is also difficult. Therefore, we reorganize Eq.~\ref{eq_3} to obtain the expression of the state-value function in terms of the policy:
 
\begin{equation}
    \begin{aligned}
        Q_{\pi_{\mathrm{ref}}}&([{x},y^{<t}],z)= \\
         &\beta \log\frac{\pi_{\theta}^*(z|[{x}, y^{<t}])}{\pi_{\mathrm{ref}}(z|[{x}, y^{<t}])} + \beta \log Z([{x}, y^{<t}];\beta).
    \end{aligned}\label{eq4}
\end{equation}
\subsection{BT Model Reformulation via Advantage Function}
To facilitate subsequent derivations, we first introduce the sequential KL divergence, as defined in \cref{def1}.
\begin{definition}\label{def1}
    Given two language models $\pi_{1}$ and $\pi_{2}$, with the input prompt ${x}$ and output response ${y}$, the sequential KL divergence is defined as:
    \begin{equation}
    \begin{split}
        D_{\mathrm{SeqKL}}({x},& {y};\pi_{1}\|\pi_{2})=\\
        &\sum\limits_{t=1}^TD_{\mathrm{KL}}(\pi_{1}(\cdot|[{x}, y^{<t}])\|\pi_{2}(\cdot|[{x}, y^{<t}])).
    \end{split}
    \end{equation}
\end{definition}

Given prompts ${x}$ and pairwise responses (${y}_1$, ${y}_2$), the Bradley-Terry model expresses the human preference probability.
However, since the Bradley-Terry model is formulated at the sentence level, it cannot establish a connection with the token-level mapping presented in Eq.~\ref{eq4}. Consequently, we need to derive a token-level preference model. Initiating from the Bradley-Terry model, we transform it into a token-level formulation and demonstrate its equivalence with the Regret Preference Model \cite{knox2023learning, knox2022models}, as shown in the \cref{lemma1}.
\begin{restatable}{lemma}{mylemmaone}\label{lemma1}
Given a reward function $r(\mathrm{x}, \mathrm{y})$, assuming a relationship between token-wise rewards and the reward function represented by $r({x}, {y}) = \sum_{t=1}^T\gamma^{t-1}R([{x},y^{<t}], y^t)$, we can establish the equivalence between the Bradley-Terry model and the Regret Preference Model in the task of text generation alignment, i.e.,
\begin{equation}
    {\small
    \begin{aligned}
        P_{\mathrm{BT}}&({y}_1 \succ {y}_2 |{x})=\\
        \sigma&\left(\sum_{t=1}^{T_1}\gamma^{t-1}A_{\pi}([{x},y_1^{<t}], y_1^{t}) - \sum_{t=1}^{T_2}\gamma^{t-1}A_{\pi}([{x},y_2^{<t}], y_2^{t})\right),
    \end{aligned}}\label{BT_eq_RPM}
\end{equation}
where $\sigma(x)=1/(1+exp(-x))$ is the logistic sigmoid function.
\end{restatable}
We prove this lemma in \ref{equivalence}.

In \cref{lemma1}, we assume that $r({x}, {y}) = \sum_{t=1}^T\gamma^{t-1}R([{x},y^{<t}], y^t)$. This assumption is natural in the context of RL, where $r({x}, {y})$ represents the overall reward for response ${y}$ given the prompt ${x}$. Considering text generation as a sequential decision-making problem, $r({x}, {y})$ can be viewed as the cumulative reward for the generated text.

According to the definition of the advantage function in \cref{Nota}, we can directly establish the relationship between the optimal solution in Eq.~\ref{eq4} and preference optimization objective in Eq.~\ref{BT_eq_RPM}. One intractable aspect is that the state-action function $Q_{\pi}$ depends on a partition function, which is contingent on both the input prompt ${x}$ and the output response ${y}$. This results in non-identical values of the partition function for a pair of responses (${y}_w$, ${y}_l$), specifically, $Z([{x},y_w^{<t}];\beta)\ne Z([{x},y_l^{<t}];\beta)$. As a result, we cannot employ a cancellation strategy similar to DPO, which relies on the property that the Bradley-Terry model depends only on the difference in rewards between two completions.

Fortunately, by expanding the advantage function $A_{\pi}$ and converting the state-value function $V_{\pi}$ into a form exclusively related to the state-action function $Q_{\pi}$, we can offset the partition function naturally. In this way, we ultimately reformulate the Bradley-Terry model to be directly tied to the optimal policy $\pi_{\theta}^*$ and the reference policy $\pi_{\mathrm{ref}}$. This is summarized in the following theorem.

\begin{restatable}{theorem}{mytheoremone}\label{theorem1}
    In the KL-constrainted advantage function maximization problem corresponding to Eq.\ref{rpo_obj}, the Bradley-Terry model express the human preference probability in terms of the optimal policy $\pi_{\theta}^*$ and reference policy $\pi_{\mathrm{ref}}$:
    \begin{equation}
    \small
        P_{\mathrm{BT}}^*({y}_1 \succ {y}_2 |{x})=\sigma(u^*({x}, {y}_1, {y}_2) - \delta^*({x}, {y}_1, {y}_2)),\label{PBT_pi}
    \end{equation}
    where, $u({x}, {y}_1, {y}_2)$ refers to the difference in rewards implicitly defined by the language model $\pi_{\theta}$ and the reference model $\pi_{\mathrm{ref}}$ \cite{rafailov2023direct}, represented as
    \begin{equation}
    \small
    u({x}, {y}_1, {y}_2)=\beta\log\frac{\pi_{\theta}({y}_1\mid {x})}{\pi_{\mathrm{ref}}({y}_1\mid {x})}-\beta\log\frac{\pi_{\theta}({y}_2\mid {x})}{\pi_{\mathrm{ref}}({y}_2\mid {x})}, \label{u_function}
    \end{equation}
    and $\delta({x}, {y}_1, {y}_2)$ refers to the difference in sequential forward KL divergence between two pairs $({x}, {y}_1)$ and $({x}, {y}_2)$, weighted by $\beta$, expressed as
    \begin{equation}
        \begin{aligned}
    \delta({x}, {y}_1, {y}_2) =&\beta D_{\mathrm{SeqKL}}\left({x},{y}_2;\pi_{\mathrm{ref}}\| \pi_{\theta}\right)\\
    &\qquad -\beta D_{\mathrm{SeqKL}}\left({x},{y}_1;\pi_{\mathrm{ref}}\| \pi_{\theta}\right).
        \end{aligned}\label{delta_function}
    \end{equation}

\end{restatable}

The proof is provided in the \cref{A_5}.

\subsection{Loss Function and Formal Analysis}\label{Loss Function}
Drawing on Eq.~\ref{PBT_pi}, we reformulate the Bradley-Terry model into a structure solely relevant to the policy. This allows us to formulate a likelihood maximization objective for a parametrized policy $\pi_{\theta}$, leading to the derivation of the loss function for the initial version of our method, $\mathrm{\methodabb{}}_1$:
\begin{align}
    &\mathcal{L}_{\mathrm{\methodabb{}_1}}\left(\pi_\theta ;\pi_{\mathrm{ref}}\right)\nonumber= \\
    &\quad -\mathbb{E}_{({x},{y}_w,{y}_l)\sim \mathcal{D}}\left[\log\sigma\left(u({x}, {y}_w, {y}_l) - \delta({x}, {y}_w, {y}_l)\right)\right].\label{eq_6}
\end{align}

Through this approach, we explicitly introduce sequential forward KL divergence into the loss function. Coupled with the implicitly integrated reverse KL divergence, we enhance our ability to balance alignment performance and generation diversity of LLMs. 

Subsequently, we conduct a derivative analysis of our method and make specific modifications to the loss function of \methodabb{}. For convenience, we use $u$ to denote $u({x}, {y}_w, {y}_l)$, and $\delta$ to represent 
$\delta({x}, {y}_w, {y}_l)$. By employing the formulation of the loss function presented in Eq.\ref{eq_6}, we compute the gradient of the loss function with respect to the parameters $\theta$: 
\begin{equation}
\begin{aligned}
    \nabla_\theta\mathcal{L}_{\mathrm{\methodabb{}_1}}&(\pi_\theta;\pi_{\mathrm{ref}})=\\
    -&\mathbb{E}_{({x},{y}_w,{y}_l)\sim\mathcal{D}}
    \left[\sigma\left(-u+\delta\right)\left[
    \nabla_\theta u
    - \nabla_{\theta}\delta\right]\right].
\end{aligned}\label{gradient}
\end{equation}

In Eq.~\ref{gradient}, $\sigma(-u+\delta)$ serves as the weighting factor for the gradient.
The first part $(-u)$ corresponds to the weight factor in the loss function of DPO. When the language model makes errors in predicting human preferences, i.e., $\log\frac{\pi_{\theta}({y}_l\mid {x})}{\pi_{\mathrm{ref}}({y}_l\mid {x})} > \log\frac{\pi_{\theta}({y}_w\mid {x})}{\pi_{\mathrm{ref}}({y}_w\mid {x})}$, the value of $(-u)$ will become larger, applying a stronger update for the comparison $({y}_w, {y}_l)$. While the second part $\delta$ is a distinctive component of our method. As shown in \cref{fig1}, the KL divergence growth rate for the dispreferred response subset is faster than that for the preferred response subset. With the increasing disparity, the corresponding value of $\delta$ rises, thereby amplifying the weight factor $\sigma(-u+\delta)$. Combined with the subsequent gradient term, our objective function can effectively suppress the difference in KL divergence between pairs of responses with large disparities in KL divergence. Through the collaborative influence of the weight factor $\delta$ and the gradient term $(-\nabla_{\theta}\delta)$, our method achieves the purpose of automatic control over the KL divergence balance.

The gradient of the loss function in Eq.~\ref{gradient} also consists of two components, $\nabla_{\theta}u$ and $(-\nabla_{\theta}\delta)$. $\nabla_{\theta}u$ represents the optimization direction of the gradient in DPO. Intuitively, $\nabla_{\theta}u$ increases the likelihood of preferred completions ${y}_w$ and decreases the likelihood of dispreferred completions ${y}_l$. While $
(-\nabla_{\theta}\delta)$ 
tends to narrow the gap between $D_{\mathrm{SeqKL}}\left({x},{y}_w;\pi_{\mathrm{ref}}\| \pi_{\theta}\right)$ and $D_{\mathrm{SeqKL}}\left({x},{y}_l;\pi_{\mathrm{ref}}\| \pi_{\theta}\right)$. 

However, when considered separately, the gradient of $D_{\mathrm{SeqKL}}\left({x},{y}_w;\pi_{\mathrm{ref}}| \pi_{\theta}\right)$ in the loss function tends to increase the sequential KL divergence between $\pi_{\mathrm{ref}}$ and $\pi_{\theta}$ at $({x}, {y}_w)$ during the optimization process. This is because the sequential forward KL divergence in the loss function is introduced through the state-value function $V_{\pi}$, inherently introducing an expectation $\mathbb{E}_{z\sim \pi_{\mathrm{ref}}}\left[\log\frac{\pi_\theta(z|[{x}, y^{<t}])}{\pi_{\mathrm{ref}}(z|[{x}, y^{<t}])}\right]$ as a baseline at each token. The negative value of this expectation corresponds precisely to a forward KL divergence $D_{\mathrm{KL}}\left(\pi_{\mathrm{ref}}(\cdot|[{x},y^{<t}]) |\pi_{\theta}(\cdot|[{x},y^{<t}])\right)$, which can be used to constrain the unbalanced growth of KL divergence. For the prompt ${x}$ and the preferred response ${y}_w$, at each token, the loss function in Eq.~\ref{gradient} tends to increase the likelihood of $\log\frac{\pi(y_w^{t}|[{x}, y_w^{<t}])}{\pi_{\mathrm{ref}}(y_w^{t}|[{x}, y_w^{<t}])}$ while simultaneously decreasing the expectation, enlarging the gap between the specified term $y_w^t$ and the baseline to expedite training. The impact of decreasing the expectation is an increase in the forward KL divergence $D_{\mathrm{KL}}\left(\pi_{\mathrm{ref}}(\cdot|[{x},y_w^{<t}]) |\pi_{\theta}(\cdot|[{x},y_w^{<t}])\right)$ at each token, leading to an increase in $D_{\mathrm{SeqKL}}\left({x},{y}_w;\pi_{\mathrm{ref}}| \pi_{\theta}\right)$. As we do not aim to accelerate the training speed and prefer to ensure training stability, we modify the loss function by discontinuing the gradient propagation of $D_{\mathrm{SeqKL}}\left({x},{y}_w;\pi_{\mathrm{ref}}| \pi_{\theta}\right)$ and treating it as a baseline term for alignment of $D_{\mathrm{SeqKL}}\left({x},{y}_l;\pi_{\mathrm{ref}}| \pi_{\theta}\right)$. 

Different from $D_{\mathrm{SeqKL}}\left({x},{y}_w;\pi_{\mathrm{ref}}| \pi_{\theta}\right)$, the gradient of $D_{\mathrm{SeqKL}}\left({x},{y}_l;\pi_{\mathrm{ref}}| \pi_{\theta}\right)$ tends to reduce the sequential KL divergence between $\pi_{\mathrm{ref}}$ and $\pi_{\theta}$ at $({x}, {y}_l)$. For the prompt ${x}$ and the rejected response ${y}_l$, the loss function in Eq.\ref{gradient} tends to decrease the likelihood of $\log\frac{\pi(y_l^{t}|[{x}, y_l^{<t}])}{\pi_{\mathrm{ref}}(y_l^{t}|[{x}, y_l^{<t}])}$ at each token while increasing the expectation $\mathbb{E}_{z\sim \pi_{\mathrm{ref}}}\left[\log\frac{\pi_\theta(z|[{x}, y_l^{<t}])}{\pi_{\mathrm{ref}}(z|[{x}, y_l^{<t}])}\right]$. The increase in the expectation 
implies a smaller forward KL divergence at that token, thereby acting to constrain the growth rate of sequential forward KL divergence. Therefore, for this term, we choose to retain its gradient updates.

In conclusion, we only propagate the gradient of the $D_{\mathrm{SeqKL}}\left({x},{y}_l;\pi_{\mathrm{ref}}| \pi_{\theta}\right)$ in $(-\nabla_{\theta}\delta)$. 
When the second part weight factor $\delta$ becomes larger, it imposes a stronger suppression on $D_{\mathrm{SeqKL}}\left({x},{y}_l;\pi_{\mathrm{ref}}\| \pi_{\theta}\right)$ to control the balance of KL divergence.

Furthermore, to achieve a better balance between alignment performance and generation diversity in \methodabb{}, we introduce an additional parameter $\alpha$ into the loss function. By adjusting the magnitude of $\alpha$, we can control the deviation between $D_{\mathrm{SeqKL}}\left({x},{y}_w;\pi_{\mathrm{ref}}\| \pi_{\theta}\right)$ and $D_{\mathrm{SeqKL}}\left({x},{y}_l;\pi_{\mathrm{ref}}\| \pi_{\theta}\right)$.

In summary, we modify the loss function of $\mathrm{\methodabb{}}_1$, resulting in the second version of our method, $\mathrm{\methodabb{}}_2$, as follows:
\begin{equation} 
    \begin{aligned}
    &\mathcal{L}_{\mathrm{\methodabb{}_2}}\left(\pi_\theta ;\pi_{\mathrm{ref}}\right)= \\
        &\ -\mathbb{E}_{({x},{y}_w,{y}_l)\sim \mathcal{D}}\left[\log\sigma\left(u({x}, {y}_w, {y}_l) - \alpha \delta_2({x}, {y}_w, {y}_l)\right)\right],
    \end{aligned}\label{tdpo2}
\end{equation}
where $\alpha$ is a parameter, and
\begin{equation}
\begin{aligned}
    \delta_2({x}, {y}_1, {y}_2) =& \beta D_{\mathrm{SeqKL}}\left({x},{y}_2;\pi_{\mathrm{ref}}\| \pi_{\theta}\right)\\
    &\qquad -\mathnormal{sg} \left(\beta D_{\mathrm{SeqKL}}\left({x},{y}_1;\pi_{\mathrm{ref}}\| \pi_{\theta}\right)\right).
\end{aligned}\label{delta2_function}
\end{equation}
The $\mathnormal{sg}$ represents the stop-gradient operator, which blocks the propagation of gradients.

We summarize the comparison of the loss functions for $\mathrm{DPO}$, $\mathrm{TDPO}_1$, and $\mathrm{TDPO}_2$, as presented in \cref{figs:tdpo}.

\begin{figure}[ht]
\vskip 0.2in
\begin{center}
\centerline{\includegraphics[width=0.85\columnwidth]{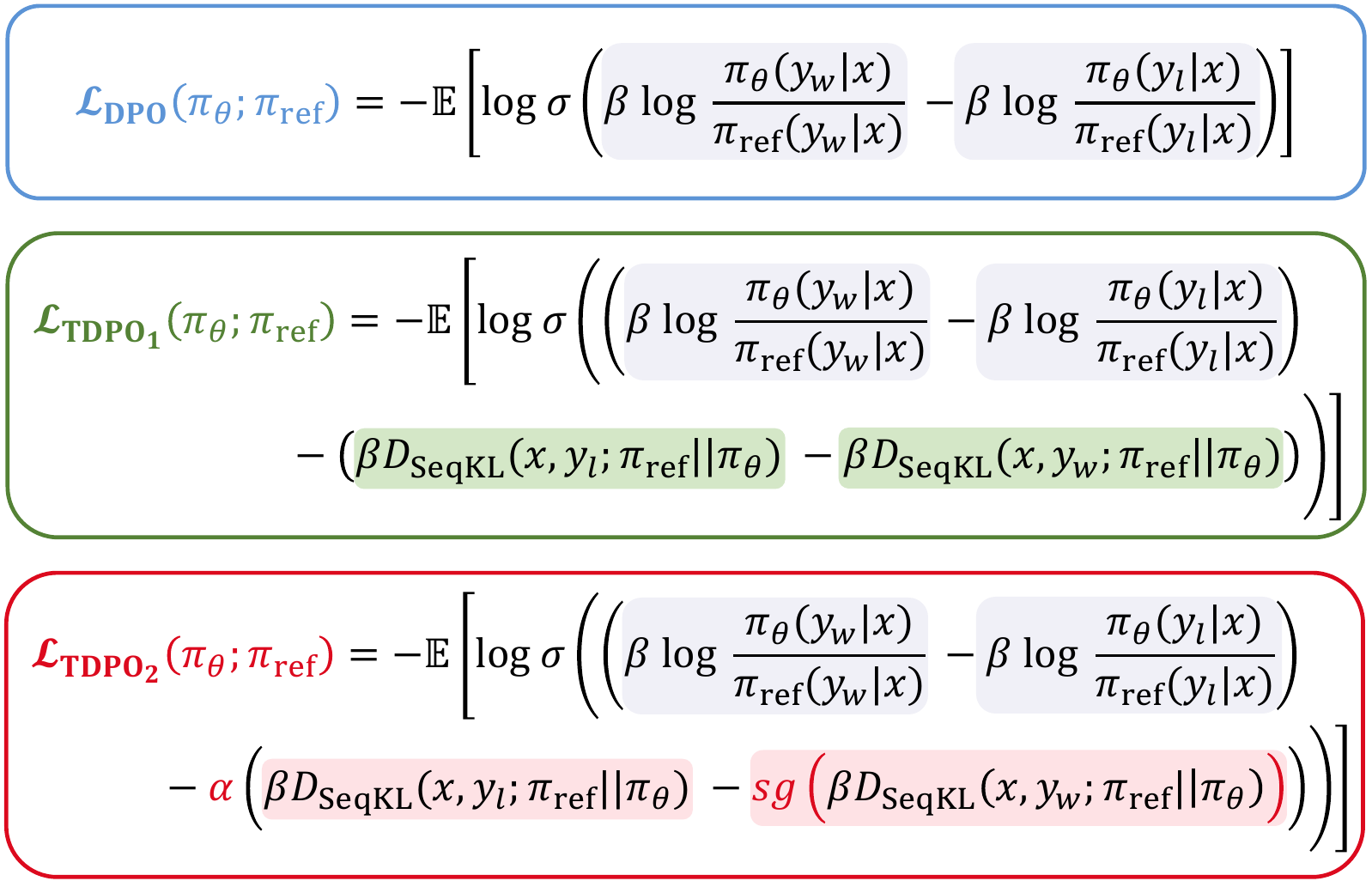}}
\caption{Comparison of Loss Functions for $\mathrm{DPO}$, $\mathrm{TDPO}_1$ and $\mathrm{TDPO}_2$ Methods. The $\mathnormal{sg}$ denotes the stop-gradient operator. Both $\mathrm{TDPO}_1$ and $\mathrm{TDPO}_2$ incorporate an additional term for finer-grained control over the KL divergence, compared to DPO.}
\label{figs:tdpo}
\end{center}
\vskip -0.1in
\end{figure}

\begin{algorithm*}[th]
   \caption{Token-level Direct Preference Optimization (TDPO)}
   \label{alg_tdpo}
\begin{algorithmic}[1]
   \STATE {\bfseries Input:} Reference model $\pi_{\mathrm{ref}}$, Policy model $\pi_{\theta}$, Coefficient $\alpha,\beta$, Learning rate $\eta$
   \STATE {\bfseries Input:} Dataset $\mathcal{D}=\{(x,y_w,y_l)^{i}\}_{i=1}^N$ of size $N$, Method $\mathcal{M}$
   \STATE {\bfseries Initialize:} $\pi_{\theta}\gets\pi_{\mathrm{ref}}$
   \FOR{each epoch}
   \STATE Sample mini-batch $\mathcal{D}_m=\{(x,y_w,y_l)^{m}\}_{m=1}^M$ from $\mathcal{D}$
   \STATE Predict the probabilities $\pi_{\theta}(y_w|x)$ and $\pi_{\theta}(y_l|x)$ for $(x, y_w, y_l)$ in the mini-batch $\mathcal{D}_m$ using the policy model
   \STATE Predict the probabilities $\pi_{\mathrm{ref}}(y_w|x)$ and $\pi_{\mathrm{ref}}(y_l|x)$ for $(x, y_w, y_l)$ in the mini-batch $\mathcal{D}_m$ using the reference model
    \STATE Calculate the function $u({x}, {y}_w, {y}_l)=\beta\log\frac{\pi_{\theta}({y}_w\mid {x})}{\pi_{\mathrm{ref}}({y}_w\mid {x})}-\beta\log\frac{\pi_{\theta}({y}_l\mid {x})}{\pi_{\mathrm{ref}}({y}_l\mid {x})}$ 
    \hfill $\triangleright$ Eq.\ref{u_function}
    \STATE Compute the sequential KL divergence $D_{\mathrm{SeqKL}}\left({x},{y}_w;\pi_{\mathrm{ref}}\| \pi_{\theta}\right)$ for $(x, y_w)$ in the mini-batch $\mathcal{D}_m$ 
    \STATE Compute the sequential KL divergence $D_{\mathrm{SeqKL}}\left({x},{y}_l;\pi_{\mathrm{ref}}\| \pi_{\theta}\right)$ for $(x, y_l)$ in the mini-batch $\mathcal{D}_m$ 
    \IF{Method $\mathcal{M}$ is $\mathrm{\methodabb{}}_1$}
    \STATE Calculate the function $\delta({x}, {y}_w, {y}_l) =\beta D_{\mathrm{SeqKL}}\left({x},{y}_l;\pi_{\mathrm{ref}}\| \pi_{\theta}\right) - \beta D_{\mathrm{SeqKL}}\left({x},{y}_w;\pi_{\mathrm{ref}}\| \pi_{\theta}\right)$ \hfill $\triangleright$ Eq.\ref{delta_function}
    \STATE $\theta \gets \theta + \eta \nabla_{\theta}\mathbb{E}_{({x},{y}_w,{y}_l)\sim \mathcal{D}_m}\left[\log\sigma\left(u({x}, {y}_w, {y}_l) - \delta({x}, {y}_w, {y}_l)\right)\right]$ \hfill $\triangleright$ Eq.\ref{eq_6}
    \ELSIF{Method $\mathcal{M}$ is $\mathrm{\methodabb{}}_2$}
    \STATE Calculate the function $\delta_2({x}, {y}_w, {y}_l) =\beta D_{\mathrm{SeqKL}}\left({x},{y}_l;\pi_{\mathrm{ref}}\| \pi_{\theta}\right)-\mathnormal{sg} \left(\beta D_{\mathrm{SeqKL}}\left({x},{y}_w;\pi_{\mathrm{ref}}\| \pi_{\theta}\right)\right)$ \hfill $\triangleright$ Eq.\ref{delta2_function} 
    \STATE $\theta \gets \theta + \eta \nabla_{\theta}\mathbb{E}_{({x},{y}_w,{y}_l)\sim \mathcal{D}_m}\left[\log\sigma\left(u({x}, {y}_w, {y}_l) - \alpha \delta_2({x}, {y}_w, {y}_l)\right)\right]$ \hfill $\triangleright$ Eq.\ref{tdpo2}
    \ENDIF
   \ENDFOR
   \STATE {\bfseries Output:} $\pi_{\theta}$
\end{algorithmic}
\end{algorithm*}

\begin{figure*}[ht]
\vskip 0.2in
\centering
\subfigure[\label{IMDb:subfig1}]{\includegraphics[width=0.38\textwidth]{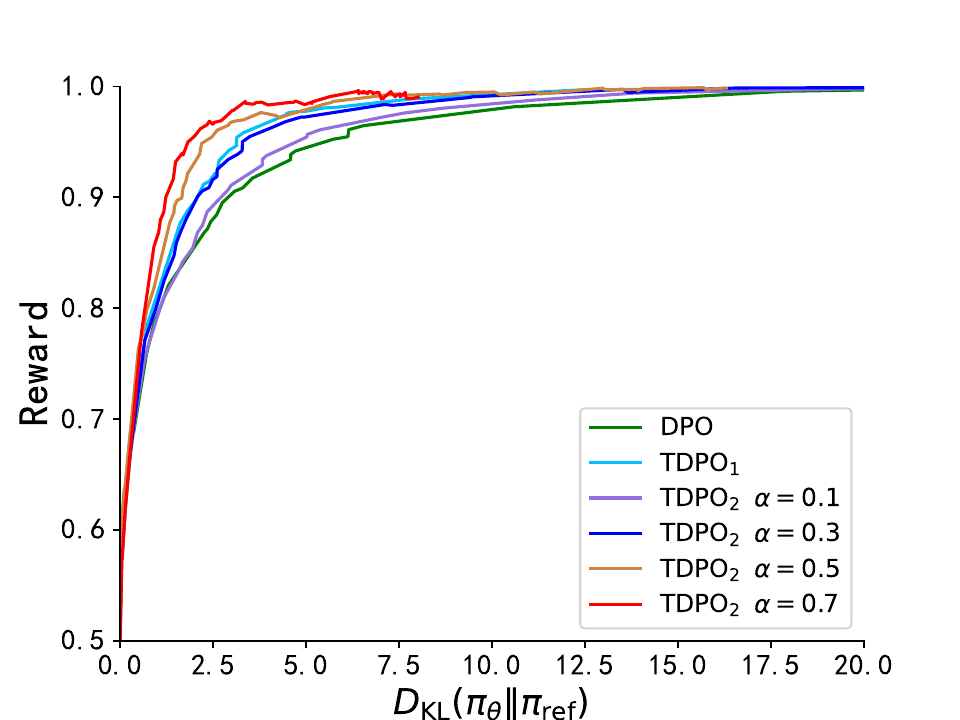}}
\subfigure[\label{IMDb:subfig2}]{\includegraphics[width=0.38\textwidth]{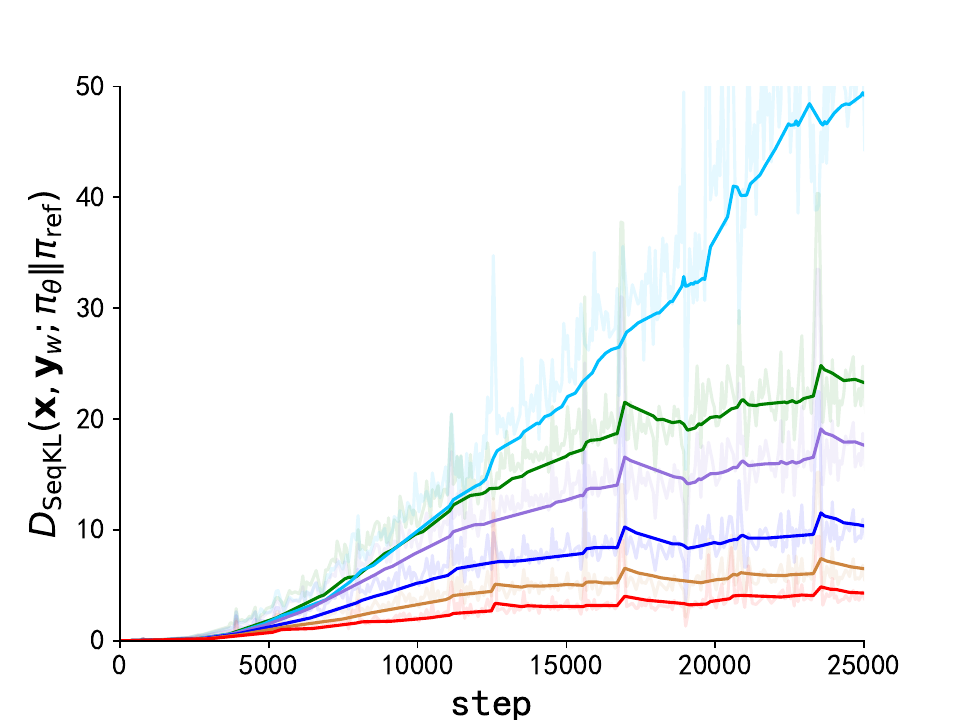}}
\subfigure[\label{IMDb:subfig3}]{\includegraphics[width=0.38\textwidth]{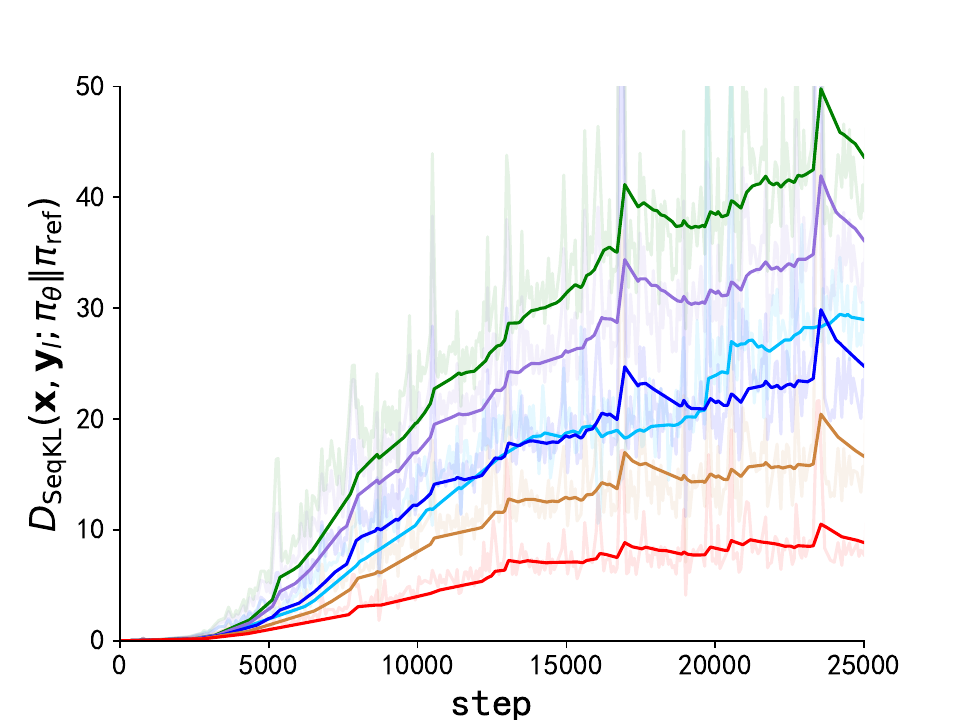}}
\subfigure[\label{IMDb:subfig4}]{\includegraphics[width=0.38\textwidth]{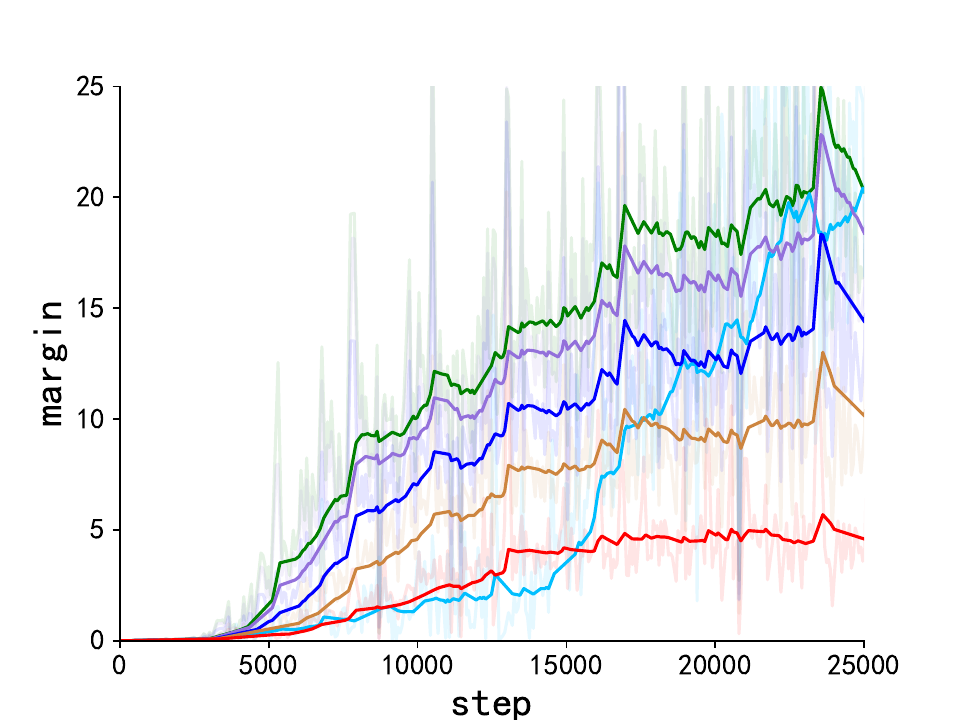}}
  \caption{The experiment on IMDb dataset. \cref{IMDb:subfig1} represents the frontier of expected reward and KL divergence with respect to the reference model. We implemented DPO, $\mathrm{\methodabb{}}_1$, and different versions of $\mathrm{\methodabb{}}_2$ with respect to the parameter $\alpha$. Both $\mathrm{\methodabb{}}_1$ and $\mathrm{\methodabb{}}_2$ outperform DPO in terms of the frontier, with $\mathrm{\methodabb{}}_2$ showing further improvement over $\mathrm{\methodabb{}}_1$. This demonstrates the effectiveness of our analysis and modifications. \cref{IMDb:subfig2} and \cref{IMDb:subfig3} present the progression of sequential KL divergence on the preferred and dispreferred responses subset over training steps respectively. \cref{IMDb:subfig4} illustrates the difference between the sequential KL divergence on the dispreferred responses subset and that on the preferred responses subset throughout the training process, namely $margin=|D_{\mathrm{SeqKL}}({x}, {y}_w;\pi_{\mathrm{ref}}\|\pi_{\theta}) - D_{\mathrm{SeqKL}}({x}, {y}_l;\pi_{\mathrm{ref}}\|\pi_{\theta})|$. $\mathrm{\methodabb{}}_2$ exhibit superior regulation over KL divergence compared to the $\mathrm{\methodabb{}}_1$ and DPO algorithm.}
  \label{fig2}
\vskip -0.1in
\end{figure*}

Leveraging the parameter $\beta$ to regulate the deviation of the language model from the base reference model, and $\alpha$ to control the balance of sequential KL divergence within the language model, our approach achieves superior alignment with human preferences while preserving model generation diversity effectively.  
We provided the pseudocode in \cref{alg_tdpo} and the Pytorch implementation version of TDPO loss in \cref{appendix_B}.

\section{Experiments}
In this section, we demonstrate the superior performance of our algorithm in three different open-sourced datasets: the IMDb sentiment dataset \cite{maas2011learning}, the Anthropic HH dataset \cite{bai2022training}, and MT-bench \cite{zheng2023judging}. The IMDb dataset serves as a controlled semantic generation dataset where the model is presented with prompts consisting of prefixes from movie reviews, and required to generate responses with positive sentiment. The Anthropic HH dataset is a single-turn dialogue dataset where the model receives human queries, covering various topics such as academic questions or life guidance. The trained model is tasked with providing helpful answers to these questions while avoiding toxic responses. 
Finally, MT-Bench is a GPT-4-based evaluation benchmark, assessing the proficiency of LLMs in handling multi-turn open-ended questions. 
Questions in MT-Bench span eight distinct knowledge domains, from areas such as writing, mathematical reasoning, and humanities. Experimental results demonstrate that MT-Bench achieves consistency with human preferences exceeding 80\%. 

\subsection{Experiments on IMDb Dataset}
In this experiment, besides our proposed methods $\mathrm{\methodabb{}}_1$ and $\mathrm{\methodabb{}}_2$, we also implemented the DPO algorithm for fair comparison. We employed GPT-2 Large \cite{radford2019language} as our base model and the model checkpoint: \textit{insub/gpt2-large-IMDb-fine-tuned}\footnote{\url{https://huggingface.co/insub/gpt2-large-IMDb-fine-tuned}} as the SFT model. 
During the evaluation, we utilized the pre-trained sentiment classifier \textit{siebert/sentiment-roberta-large-english}\footnote{\url{https://huggingface.co/siebert/sentiment-roberta-large-english}} to compute rewards. 
For DPO, we followed the official implementation \cite{rafailov2023direct}, setting $\beta$ at 0.1. 
To analyze the effectiveness of each algorithm in optimizing the constrained reward maximization objective, we evaluated each algorithm after 100 training steps until convergence, computing its frontier of average reward and average sequential KL divergence with the reference policy.

The results are depicted in \cref{IMDb:subfig1}. We implement the DPO, $\mathrm{\methodabb{}}_1$, and different versions of $\mathrm{\methodabb{}}_2$ algorithms with varying the parameter $\alpha$. From the figure, we notice that although DPO establishes an efficient frontier, $\mathrm{\methodabb{}}_1$ and $\mathrm{\methodabb{}}_2$ outperform DPO in terms of divergence versus reward on the frontier, achieving higher reward while maintaining low KL divergence. We also implemented versions of $\mathrm{\methodabb{}}_2$ with $\alpha \in \{1, 1.5, 2, 5\}$. However, we found that higher values of $\alpha$ made it difficult to optimize the reward. In \cref{IMDb:subfig2,IMDb:subfig3,IMDb:subfig4}, we illustrate the curves portraying the sequential KL divergence for different algorithms during the training process. The sequential KL divergence growth rate of DPO on the dispreferred response subset is significantly higher than that on the preferred response subset, leading to an increasing offset between them. In contrast, $\mathrm{\methodabb{}}_2$ exhibits superior control over KL divergence, achieving better divergence efficiency compared to DPO. As analyzed in \cref{Loss Function}, $\mathrm{\methodabb{}}_1$ tends to result in an increased sequential KL divergence on the preferred response subset, thereby exhibiting a weaker capacity for KL divergence adjustment compared to $\mathrm{\methodabb{}}_2$. $\mathrm{\methodabb{}}_2$ maintains a more balanced sequential KL divergence on both dispreferred and preferred response subsets, contributing to its ability to achieve a superior frontier. Although a larger $\alpha$ enhances control over the sequential KL divergence, it also affects the speed and difficulty of optimization. For the remainder of this paper, we set $\alpha=0.5$. In \cref{appendix_C}, we also present graphs of the frontier between the reward and forward KL divergence and the progression curves of the forward KL divergence throughout the training process.

\subsection{Experiments on Anthropic HH Dataset}
\begin{table}[t]

\caption{Comparison of 
DPO, $\mathrm{\methodabb{}}_1$ and $\mathrm{\methodabb{}}_2$ in terms of the trade-off between Alignment (accuracy) and Diversity (entropy) on the Anthropic HH dataset. The $\uparrow$ indicates higher values are preferable.}
\vskip 0.2in
\begin{center}
\begin{small}
\begin{tabular}{c|c|c}
\toprule
\multirow{2}{*}{\textbf{Method}} & \textbf{Alignment} &{\textbf{Diversity}} \\
&\text{Accuracy}$(\%)$ $\uparrow$ &\text{Entropy} $\uparrow$\\
\midrule
\text{f-DPO(FKL)}& 54.71 & 4.708\\
\text{DPO}& 59.43 & 3.196\\
\text{$\mathrm{\methodabb{}}_1$(\textbf{ours})}& 60.08 & 4.727 \\
\text{$\mathrm{\methodabb{}}_2$(\textbf{ours})} & \textbf{67.33} & \textbf{4.915}\\
\bottomrule
\end{tabular}
\end{small}
\end{center}
\label{table1}
\vskip -0.2in
\end{table}
Next, we evaluate the performance of $\mathrm{\methodabb{}}_1$ and  $\mathrm{\methodabb{}}_2$ on the Anthropic HH dataset. We use Pythia-2.8B \cite{biderman2023pythia} as the base model and fine-tune the base model on chosen completions to train a reference model, such that completions are within the distribution of the model. Subsequently, we train $\mathrm{\methodabb{}}_1$,  $\mathrm{\methodabb{}}_2$, DPO \cite{rafailov2023direct} and f-DPO with forward KL divergence constraint \cite{wang2023beyond}
on this reference model. 
In this experiment, our primary focus is on two aspects: 1) the trade-off between alignment and diversity in generating responses among different algorithms, and 2) the ability of different algorithms to align with human preferences. For the first part, we utilize automatic metrics for evaluation, while for the second part, we rely on the GPT-4 evaluation. Both evaluations were conducted on the test set of the Anthropic HH dataset.

To assess the alignment performance of different algorithms in generating responses, we compute the accuracy of generated responses relative to chosen completions in the test dataset. To measure the diversity, we employ nucleus sampling with $p=0.95$ to generate 25 responses and utilize the predictive entropy 
as the evaluation metric. The trade-off between alignment accuracy and diversity for different algorithms is summarized in \cref{table1}. $\mathrm{\methodabb{}}_2$ not only surpasses DPO, f-DPO and $\mathrm{\methodabb{}}_1$ in terms of accuracy but also excels in entropy, achieving a superior balance between alignment and diversity.

To further assess the ability of $\mathrm{\methodabb{}}_1$ and $\mathrm{\methodabb{}}_2$ to align with human preferences, we evaluated the win rates of responses generated by models trained with different algorithms against chosen responses on the test set of the HH dataset, the result is illustrated in the \cref{hh_win_rate}. Compared to the SFT model, the DPO, $\mathrm{\methodabb{}}_1$, and $\mathrm{\methodabb{}}_2$ algorithms better align with human preferences, achieving win rates not less than $50\%$ against chosen responses at temperature $0.75$. This demonstrates that both $\mathrm{\methodabb{}}_1$, and $\mathrm{\methodabb{}}_2$ possess a strong capability to align with human preferences.

\begin{figure}[ht]
\begin{center}
\centerline{\includegraphics[width=0.8\columnwidth]{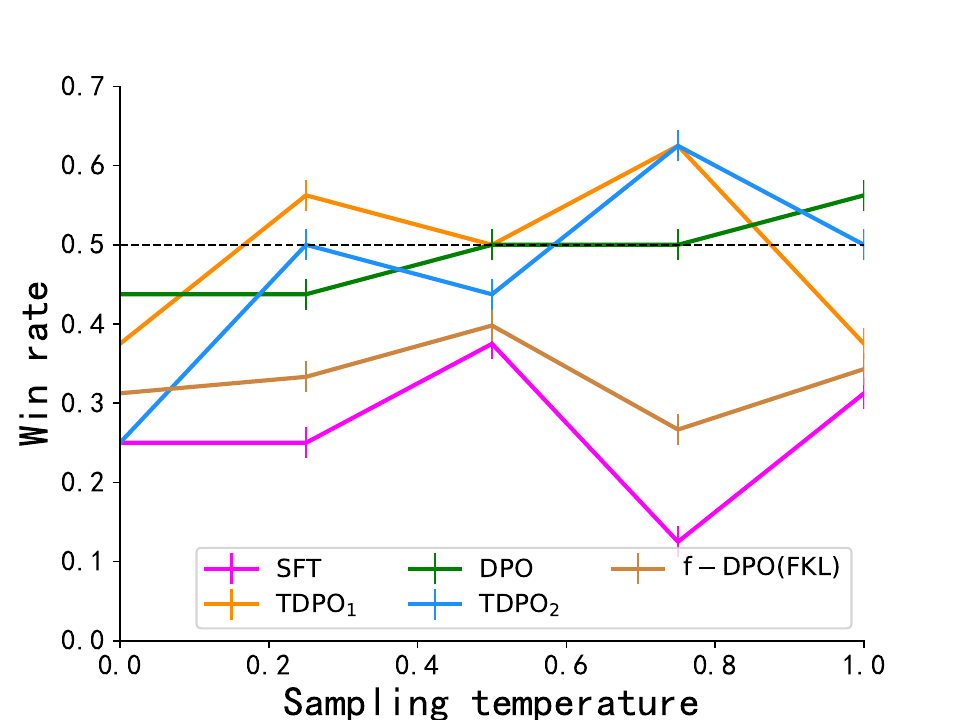}}
\caption{The win rates, computed by GPT-4, in comparison to the chosen responses for Anthropic-HH one-step dialogue.}
\label{hh_win_rate}
\end{center}
\vskip -0.2in
\end{figure}

\subsection{Experiments on MT-Bench}
To comprehensively evaluate $\mathrm{\methodabb{}}_1$, and $\mathrm{\methodabb{}}_2$ in terms of generation quality, we conducted pairwise comparisons on the MT-Bench using models trained on the Anthropic HH dataset. Following the official MT-Bench implementation, we sampled responses with a temperature coefficient of $0.7$ and constrained the maximum number of newly generated tokens to 512. For the PPO baseline, we employed the trlx framework \cite{havrilla2023trlx}, utilizing the proxy reward model \textit{Dahoas/gptj-rm-static}\footnote{\url{https://huggingface.co/Dahoas/gptj-rm-static}} during training. The result is depicted in the \cref{mt_bench}. It reveals that $\mathrm{TDPO}_2$ achieves a higher win rate compared to other algorithms, indicating its ability to assist LLMs in generating higher-quality responses. This advantage is attributed to its exceptional ability to regulate KL divergence, facilitating a better balance between performance alignment and generation diversity.

\begin{figure}[ht]
\begin{center}
\centerline{\includegraphics[width=0.85\columnwidth]{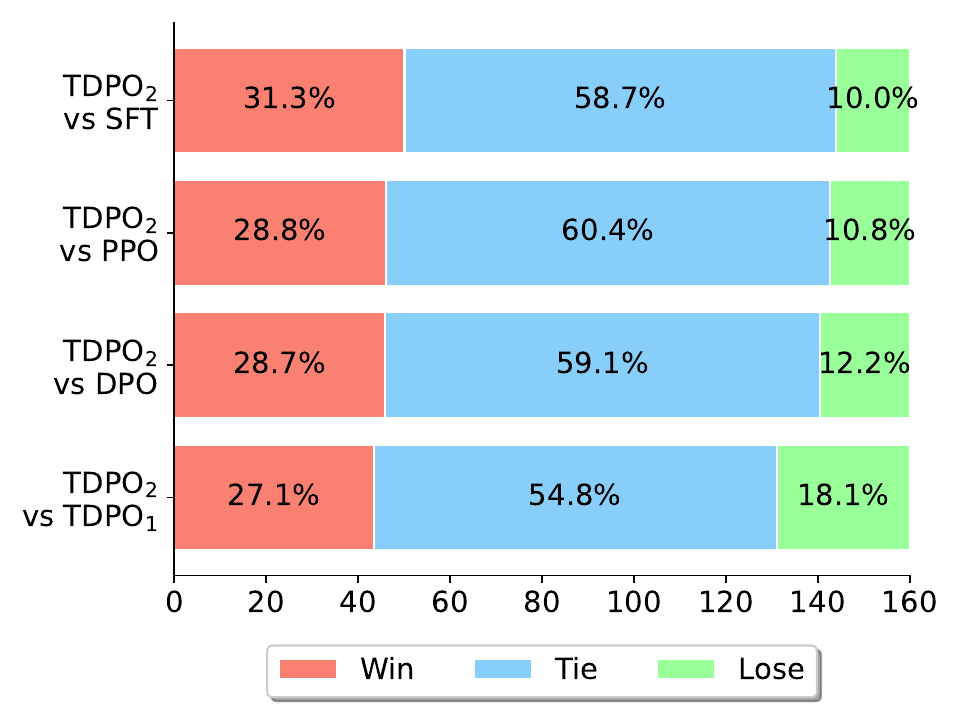}}
\caption{MT-Bench comparison between $\mathrm{SFT}$, $\mathrm{PPO}$, $\mathrm{DPO}$, $\mathrm{TDPO}_1$ and $\mathrm{TDPO}_2$ methods. The win, tie, and lose rates are evaluated based on GPT-4.}
\label{mt_bench}
\end{center}
\vskip -0.2in
\end{figure}

\section{Conclusion}
In this work, we introduced \methodname{} (\methodabb{}), an innovative token-level fine-tuning approach for Large Language Models (LLMs) aimed at aligning more closely with human preferences. By employing the token-wise optimization with forward KL divergence constraints and converting the Bradley-Terry model into a token-level preference model, \methodabb{} addresses key challenges in divergence efficiency and content diversity, surpassing traditional methods like Direct Preference Optimization (DPO) and PPO-based RLHF in tasks such as controlled sentiment generation and single-turn dialogues. This marks a substantial advancement in LLM training methodologies, demonstrating the potential of token-level optimization to enhance the alignment, quality, and diversity of LLM outputs, setting a new direction for AI alignment research and the development of nuanced, human-aligned AI systems.


Regarding the future prospects of alignment methodologies, we anticipate that iterative refinement approaches and multi-turn conversational alignment strategies will significantly improve the alignment of large language models with human values. By continuously refining these models, we can achieve more precise alignment with complex human preferences. Moreover, multi-turn conversations enable deeper and more nuanced interactions, fostering comprehensive attunement to human intentions. These approaches aim to enhance the quality and relevance of AI responses, making AI systems more harmonized with human values and expectations.




\section*{Acknowledgements}
The research leading to these results received funding from National Key R\&D Program of China (2022ZD0116402). In addition, it received funding from Science and Technology Research and Development Project of China State Railway Group Corporation Limited (P2022X012).

\section*{Impact Statement}
This paper presents work whose goal is to advance the field of Machine Learning. There are many potential societal consequences of our work, none which we feel must be specifically highlighted here.

\nocite{langley00}

\bibliography{example_paper}

\begin{thebibliography}{49}
\providecommand{\natexlab}[1]{#1}
\providecommand{\url}[1]{\texttt{#1}}
\expandafter\ifx\csname urlstyle\endcsname\relax
  \providecommand{\doi}[1]{doi: #1}\else
  \providecommand{\doi}{doi: \begingroup \urlstyle{rm}\Url}\fi

\bibitem[Achiam et~al.(2023)Achiam, Adler, Agarwal, Ahmad, Akkaya, Aleman,
  Almeida, Altenschmidt, Altman, Anadkat, et~al.]{achiam2023gpt}
Achiam, J., Adler, S., Agarwal, S., Ahmad, L., Akkaya, I., Aleman, F.~L.,
  Almeida, D., Altenschmidt, J., Altman, S., Anadkat, S., et~al.
\newblock Gpt-4 technical report.
\newblock \emph{arXiv preprint arXiv:2303.08774}, 2023.

\bibitem[Bai et~al.(2022)Bai, Jones, Ndousse, Askell, Chen, DasSarma, Drain,
  Fort, Ganguli, Henighan, et~al.]{bai2022training}
Bai, Y., Jones, A., Ndousse, K., Askell, A., Chen, A., DasSarma, N., Drain, D.,
  Fort, S., Ganguli, D., Henighan, T., et~al.
\newblock Training a helpful and harmless assistant with reinforcement learning
  from human feedback.
\newblock \emph{arXiv preprint arXiv:2204.05862}, 2022.

\bibitem[Biderman et~al.(2023)Biderman, Schoelkopf, Anthony, Bradley,
  O’Brien, Hallahan, Khan, Purohit, Prashanth, Raff,
  et~al.]{biderman2023pythia}
Biderman, S., Schoelkopf, H., Anthony, Q.~G., Bradley, H., O’Brien, K.,
  Hallahan, E., Khan, M.~A., Purohit, S., Prashanth, U.~S., Raff, E., et~al.
\newblock Pythia: A suite for analyzing large language models across training
  and scaling.
\newblock In \emph{International Conference on Machine Learning}, pp.\
  2397--2430. PMLR, 2023.

\bibitem[Bradley \& Terry(1952)Bradley and Terry]{bradley1952rank}
Bradley, R.~A. and Terry, M.~E.
\newblock Rank analysis of incomplete block designs: I. the method of paired
  comparisons.
\newblock \emph{Biometrika}, 39\penalty0 (3/4):\penalty0 324--345, 1952.

\bibitem[Brown et~al.(2020)Brown, Mann, Ryder, Subbiah, Kaplan, Dhariwal,
  Neelakantan, Shyam, Sastry, Askell, et~al.]{brown2020language}
Brown, T., Mann, B., Ryder, N., Subbiah, M., Kaplan, J.~D., Dhariwal, P.,
  Neelakantan, A., Shyam, P., Sastry, G., Askell, A., et~al.
\newblock Language models are few-shot learners.
\newblock \emph{Advances in neural information processing systems},
  33:\penalty0 1877--1901, 2020.

\bibitem[Bubeck et~al.(2023)Bubeck, Chandrasekaran, Eldan, Gehrke, Horvitz,
  Kamar, Lee, Lee, Li, Lundberg, et~al.]{bubeck2023sparks}
Bubeck, S., Chandrasekaran, V., Eldan, R., Gehrke, J., Horvitz, E., Kamar, E.,
  Lee, P., Lee, Y.~T., Li, Y., Lundberg, S., et~al.
\newblock Sparks of artificial general intelligence: Early experiments with
  gpt-4.
\newblock \emph{arXiv preprint arXiv:2303.12712}, 2023.

\bibitem[Chen et~al.(2021)Chen, Tworek, Jun, Yuan, de~Oliveira~Pinto, Kaplan,
  Edwards, Burda, Joseph, Brockman, Ray, Puri, Krueger, Petrov, Khlaaf, Sastry,
  Mishkin, Chan, Gray, Ryder, Pavlov, Power, Kaiser, Bavarian, Winter, Tillet,
  Such, Cummings, Plappert, Chantzis, Barnes, Herbert-Voss, Guss, Nichol,
  Paino, Tezak, Tang, Babuschkin, Balaji, Jain, Saunders, Hesse, Carr, Leike,
  Achiam, Misra, Morikawa, Radford, Knight, Brundage, Murati, Mayer, Welinder,
  McGrew, Amodei, McCandlish, Sutskever, and Zaremba]{chen2021evaluating}
Chen, M., Tworek, J., Jun, H., Yuan, Q., de~Oliveira~Pinto, H.~P., Kaplan, J.,
  Edwards, H., Burda, Y., Joseph, N., Brockman, G., Ray, A., Puri, R., Krueger,
  G., Petrov, M., Khlaaf, H., Sastry, G., Mishkin, P., Chan, B., Gray, S.,
  Ryder, N., Pavlov, M., Power, A., Kaiser, L., Bavarian, M., Winter, C.,
  Tillet, P., Such, F.~P., Cummings, D., Plappert, M., Chantzis, F., Barnes,
  E., Herbert-Voss, A., Guss, W.~H., Nichol, A., Paino, A., Tezak, N., Tang,
  J., Babuschkin, I., Balaji, S., Jain, S., Saunders, W., Hesse, C., Carr,
  A.~N., Leike, J., Achiam, J., Misra, V., Morikawa, E., Radford, A., Knight,
  M., Brundage, M., Murati, M., Mayer, K., Welinder, P., McGrew, B., Amodei,
  D., McCandlish, S., Sutskever, I., and Zaremba, W.
\newblock Evaluating large language models trained on code, 2021.

\bibitem[Chiang et~al.(2023)Chiang, Li, Lin, Sheng, Wu, Zhang, Zheng, Zhuang,
  Zhuang, Gonzalez, et~al.]{chiang2023vicuna}
Chiang, W.-L., Li, Z., Lin, Z., Sheng, Y., Wu, Z., Zhang, H., Zheng, L.,
  Zhuang, S., Zhuang, Y., Gonzalez, J.~E., et~al.
\newblock Vicuna: An open-source chatbot impressing gpt-4 with 90\%* chatgpt
  quality.
\newblock \emph{See https://vicuna. lmsys. org (accessed 14 April 2023)}, 2023.

\bibitem[Christiano et~al.(2017)Christiano, Leike, Brown, Martic, Legg, and
  Amodei]{christiano2017deep}
Christiano, P.~F., Leike, J., Brown, T., Martic, M., Legg, S., and Amodei, D.
\newblock Deep reinforcement learning from human preferences.
\newblock \emph{Advances in neural information processing systems}, 30, 2017.

\bibitem[Chung et~al.(2022)Chung, Hou, Longpre, Zoph, Tay, Fedus, Li, Wang,
  Dehghani, Brahma, et~al.]{chung2022scaling}
Chung, H.~W., Hou, L., Longpre, S., Zoph, B., Tay, Y., Fedus, W., Li, Y., Wang,
  X., Dehghani, M., Brahma, S., et~al.
\newblock Scaling instruction-finetuned language models.
\newblock \emph{arXiv preprint arXiv:2210.11416}, 2022.

\bibitem[Dong et~al.(2023)Dong, Xiong, Goyal, Pan, Diao, Zhang, Shum, and
  Zhang]{dong2023raft}
Dong, H., Xiong, W., Goyal, D., Pan, R., Diao, S., Zhang, J., Shum, K., and
  Zhang, T.
\newblock Raft: Reward ranked finetuning for generative foundation model
  alignment.
\newblock \emph{arXiv preprint arXiv:2304.06767}, 2023.

\bibitem[Ganguli et~al.(2022)Ganguli, Lovitt, Kernion, Askell, Bai, Kadavath,
  Mann, Perez, Schiefer, Ndousse, et~al.]{ganguli2022red}
Ganguli, D., Lovitt, L., Kernion, J., Askell, A., Bai, Y., Kadavath, S., Mann,
  B., Perez, E., Schiefer, N., Ndousse, K., et~al.
\newblock Red teaming language models to reduce harms: Methods, scaling
  behaviors, and lessons learned.
\newblock \emph{arXiv preprint arXiv:2209.07858}, 2022.

\bibitem[Gao et~al.(2023)Gao, Madaan, Zhou, Alon, Liu, Yang, Callan, and
  Neubig]{gao2023pal}
Gao, L., Madaan, A., Zhou, S., Alon, U., Liu, P., Yang, Y., Callan, J., and
  Neubig, G.
\newblock Pal: Program-aided language models, 2023.

\bibitem[Glaese et~al.(2022)Glaese, McAleese, Trębacz, Aslanides, Firoiu,
  Ewalds, Rauh, Weidinger, Chadwick, Thacker, et~al.]{glaese2022improving}
Glaese, A., McAleese, N., Trębacz, M., Aslanides, J., Firoiu, V., Ewalds, T.,
  Rauh, M., Weidinger, L., Chadwick, M., Thacker, P., et~al.
\newblock Improving alignment of dialogue agents via targeted human judgements.
\newblock \emph{arXiv preprint arXiv:2209.14375}, 2022.

\bibitem[Havrilla et~al.(2023)Havrilla, Zhuravinskyi, Phung, Tiwari, Tow,
  Biderman, Anthony, and Castricato]{havrilla2023trlx}
Havrilla, A., Zhuravinskyi, M., Phung, D., Tiwari, A., Tow, J., Biderman, S.,
  Anthony, Q., and Castricato, L.
\newblock trlx: A framework for large scale reinforcement learning from human
  feedback.
\newblock In \emph{Proceedings of the 2023 Conference on Empirical Methods in
  Natural Language Processing}, pp.\  8578--8595, 2023.

\bibitem[Huang et~al.(2023)Huang, Yu, Ma, Zhong, Feng, Wang, Chen, Peng, Feng,
  Qin, et~al.]{huang2023survey}
Huang, L., Yu, W., Ma, W., Zhong, W., Feng, Z., Wang, H., Chen, Q., Peng, W.,
  Feng, X., Qin, B., et~al.
\newblock A survey on hallucination in large language models: Principles,
  taxonomy, challenges, and open questions.
\newblock \emph{arXiv preprint arXiv:2311.05232}, 2023.

\bibitem[Jiang et~al.(2023)Jiang, Sablayrolles, Mensch, Bamford, Chaplot,
  de~las Casas, Bressand, Lengyel, Lample, Saulnier, Lavaud, Lachaux, Stock,
  Scao, Lavril, Wang, Lacroix, and Sayed]{jiang2023mistral}
Jiang, A.~Q., Sablayrolles, A., Mensch, A., Bamford, C., Chaplot, D.~S., de~las
  Casas, D., Bressand, F., Lengyel, G., Lample, G., Saulnier, L., Lavaud,
  L.~R., Lachaux, M.-A., Stock, P., Scao, T.~L., Lavril, T., Wang, T., Lacroix,
  T., and Sayed, W.~E.
\newblock Mistral 7b, 2023.

\bibitem[Khalifa et~al.(2020)Khalifa, Elsahar, and
  Dymetman]{khalifa2020distributional}
Khalifa, M., Elsahar, H., and Dymetman, M.
\newblock A distributional approach to controlled text generation.
\newblock \emph{arXiv preprint arXiv:2012.11635}, 2020.

\bibitem[Knox et~al.(2022)Knox, Hatgis-Kessell, Booth, Niekum, Stone, and
  Allievi]{knox2022models}
Knox, W.~B., Hatgis-Kessell, S., Booth, S., Niekum, S., Stone, P., and Allievi,
  A.
\newblock Models of human preference for learning reward functions.
\newblock \emph{arXiv preprint arXiv:2206.02231}, 2022.

\bibitem[Knox et~al.(2023)Knox, Hatgis-Kessell, Adalgeirsson, Booth, Dragan,
  Stone, and Niekum]{knox2023learning}
Knox, W.~B., Hatgis-Kessell, S., Adalgeirsson, S.~O., Booth, S., Dragan, A.,
  Stone, P., and Niekum, S.
\newblock Learning optimal advantage from preferences and mistaking it for
  reward.
\newblock \emph{arXiv preprint arXiv:2310.02456}, 2023.

\bibitem[Koh et~al.(2022)Koh, Ju, Liu, and Pan]{Koh_2022}
Koh, H.~Y., Ju, J., Liu, M., and Pan, S.
\newblock An empirical survey on long document summarization: Datasets, models,
  and metrics.
\newblock \emph{ACM Computing Surveys}, 55\penalty0 (8):\penalty0 1–35,
  December 2022.
\newblock ISSN 1557-7341.
\newblock \doi{10.1145/3545176}.
\newblock URL \url{http://dx.doi.org/10.1145/3545176}.

\bibitem[Langley(2000)]{langley00}
Langley, P.
\newblock Crafting papers on machine learning.
\newblock In Langley, P. (ed.), \emph{Proceedings of the 17th International
  Conference on Machine Learning (ICML 2000)}, pp.\  1207--1216, Stanford, CA,
  2000. Morgan Kaufmann.

\bibitem[Liu et~al.(2023)Liu, Zhao, Joshi, Khalman, Saleh, Liu, and
  Liu]{liu2023statistical}
Liu, T., Zhao, Y., Joshi, R., Khalman, M., Saleh, M., Liu, P.~J., and Liu, J.
\newblock Statistical rejection sampling improves preference optimization.
\newblock \emph{arXiv preprint arXiv:2309.06657}, 2023.

\bibitem[Maas et~al.(2011)Maas, Daly, Pham, Huang, Ng, and
  Potts]{maas2011learning}
Maas, A., Daly, R.~E., Pham, P.~T., Huang, D., Ng, A.~Y., and Potts, C.
\newblock Learning word vectors for sentiment analysis.
\newblock In \emph{Proceedings of the 49th annual meeting of the association
  for computational linguistics: Human language technologies}, pp.\  142--150,
  2011.

\bibitem[Ouyang et~al.(2022)Ouyang, Wu, Jiang, Almeida, Wainwright, Mishkin,
  Zhang, Agarwal, Slama, Ray, et~al.]{ouyang2022training}
Ouyang, L., Wu, J., Jiang, X., Almeida, D., Wainwright, C., Mishkin, P., Zhang,
  C., Agarwal, S., Slama, K., Ray, A., et~al.
\newblock Training language models to follow instructions with human feedback.
\newblock \emph{Advances in Neural Information Processing Systems},
  35:\penalty0 27730--27744, 2022.

\bibitem[Perez et~al.(2022)Perez, Huang, Song, Cai, Ring, Aslanides, Glaese,
  McAleese, and Irving]{perez2202red}
Perez, E., Huang, S., Song, F., Cai, T., Ring, R., Aslanides, J., Glaese, A.,
  McAleese, N., and Irving, G.
\newblock Red teaming language models with language models, 2022.
\newblock \emph{URL https://arxiv. org/abs/2202.03286}, 2022.

\bibitem[Puterman(2014)]{puterman2014markov}
Puterman, M.~L.
\newblock \emph{Markov decision processes: discrete stochastic dynamic
  programming}.
\newblock John Wiley \& Sons, 2014.

\bibitem[Radford et~al.(2019)Radford, Wu, Child, Luan, Amodei, Sutskever,
  et~al.]{radford2019language}
Radford, A., Wu, J., Child, R., Luan, D., Amodei, D., Sutskever, I., et~al.
\newblock Language models are unsupervised multitask learners.
\newblock \emph{OpenAI blog}, 1\penalty0 (8):\penalty0 9, 2019.

\bibitem[Rafailov et~al.(2023)Rafailov, Sharma, Mitchell, Ermon, Manning, and
  Finn]{rafailov2023direct}
Rafailov, R., Sharma, A., Mitchell, E., Ermon, S., Manning, C.~D., and Finn, C.
\newblock Direct preference optimization: Your language model is secretly a
  reward model.
\newblock \emph{arXiv preprint arXiv:2305.18290}, 2023.

\bibitem[Rafailov et~al.(2024)Rafailov, Hejna, Park, and Finn]{rafailov2024r}
Rafailov, R., Hejna, J., Park, R., and Finn, C.
\newblock From $ r $ to {$ Q^* $}: {Y}our {L}anguage {M}odel is {S}ecretly a
  {Q-F}unction.
\newblock \emph{arXiv preprint arXiv:2404.12358}, 2024.

\bibitem[Raffel et~al.(2020)Raffel, Shazeer, Roberts, Lee, Narang, Matena,
  Zhou, Li, and Liu]{raffel2020exploring}
Raffel, C., Shazeer, N., Roberts, A., Lee, K., Narang, S., Matena, M., Zhou,
  Y., Li, W., and Liu, P.~J.
\newblock Exploring the limits of transfer learning with a unified text-to-text
  transformer.
\newblock \emph{The Journal of Machine Learning Research}, 21\penalty0
  (1):\penalty0 5485--5551, 2020.

\bibitem[Rauh et~al.(2022)Rauh, Mellor, Uesato, Huang, Welbl, Weidinger,
  Dathathri, Glaese, Irving, Gabriel, Isaac, and
  Hendricks]{rauh2022characteristics}
Rauh, M., Mellor, J., Uesato, J., Huang, P.-S., Welbl, J., Weidinger, L.,
  Dathathri, S., Glaese, A., Irving, G., Gabriel, I., Isaac, W., and Hendricks,
  L.~A.
\newblock Characteristics of harmful text: Towards rigorous benchmarking of
  language models, 2022.

\bibitem[Schulman et~al.(2015)Schulman, Levine, Abbeel, Jordan, and
  Moritz]{schulman2015trust}
Schulman, J., Levine, S., Abbeel, P., Jordan, M., and Moritz, P.
\newblock Trust region policy optimization.
\newblock In \emph{International conference on machine learning}, pp.\
  1889--1897. PMLR, 2015.

\bibitem[Sheng et~al.(2021)Sheng, Chang, Natarajan, and
  Peng]{sheng2021societal}
Sheng, E., Chang, K.-W., Natarajan, P., and Peng, N.
\newblock Societal biases in language generation: Progress and challenges.
\newblock \emph{arXiv preprint arXiv:2105.04054}, 2021.

\bibitem[Song et~al.(2023)Song, Yu, Li, Yu, Huang, Li, and
  Wang]{song2023preference}
Song, F., Yu, B., Li, M., Yu, H., Huang, F., Li, Y., and Wang, H.
\newblock Preference ranking optimization for human alignment.
\newblock \emph{arXiv preprint arXiv:2306.17492}, 2023.

\bibitem[Stiennon et~al.(2022)Stiennon, Ouyang, Wu, Ziegler, Lowe, Voss,
  Radford, Amodei, and Christiano]{stiennon2022learning}
Stiennon, N., Ouyang, L., Wu, J., Ziegler, D.~M., Lowe, R., Voss, C., Radford,
  A., Amodei, D., and Christiano, P.
\newblock Learning to summarize from human feedback, 2022.

\bibitem[Taori et~al.(2023)Taori, Gulrajani, Zhang, Dubois, Li, Guestrin,
  Liang, and Hashimoto]{taori2023alpaca}
Taori, R., Gulrajani, I., Zhang, T., Dubois, Y., Li, X., Guestrin, C., Liang,
  P., and Hashimoto, T.~B.
\newblock Alpaca: A strong, replicable instruction-following model.
\newblock \emph{Stanford Center for Research on Foundation Models.
  https://crfm. stanford. edu/2023/03/13/alpaca. html}, 3\penalty0
  (6):\penalty0 7, 2023.

\bibitem[Team et~al.(2023)Team, Anil, Borgeaud, Wu, Alayrac, Yu, Soricut,
  Schalkwyk, Dai, Hauth, et~al.]{team2023gemini}
Team, G., Anil, R., Borgeaud, S., Wu, Y., Alayrac, J.-B., Yu, J., Soricut, R.,
  Schalkwyk, J., Dai, A.~M., Hauth, A., et~al.
\newblock Gemini: a family of highly capable multimodal models.
\newblock \emph{arXiv preprint arXiv:2312.11805}, 2023.

\bibitem[Touvron et~al.(2023)Touvron, Martin, Stone, Albert, Almahairi, Babaei,
  Bashlykov, Batra, Bhargava, Bhosale, et~al.]{touvron2023llama}
Touvron, H., Martin, L., Stone, K., Albert, P., Almahairi, A., Babaei, Y.,
  Bashlykov, N., Batra, S., Bhargava, P., Bhosale, S., et~al.
\newblock Llama 2: Open foundation and fine-tuned chat models.
\newblock \emph{arXiv preprint arXiv:2307.09288}, 2023.

\bibitem[Tunstall et~al.(2023)Tunstall, Beeching, Lambert, Rajani, Rasul,
  Belkada, Huang, von Werra, Fourrier, Habib, et~al.]{tunstall2023zephyr}
Tunstall, L., Beeching, E., Lambert, N., Rajani, N., Rasul, K., Belkada, Y.,
  Huang, S., von Werra, L., Fourrier, C., Habib, N., et~al.
\newblock Zephyr: Direct distillation of lm alignment.
\newblock \emph{arXiv preprint arXiv:2310.16944}, 2023.

\bibitem[Vu et~al.(2023)Vu, He, Haffari, and Shareghi]{vu2023koala}
Vu, T.-T., He, X., Haffari, G., and Shareghi, E.
\newblock Koala: An index for quantifying overlaps with pre-training corpora,
  2023.

\bibitem[Wang et~al.(2023)Wang, Jiang, Yang, Liu, and Chen]{wang2023beyond}
Wang, C., Jiang, Y., Yang, C., Liu, H., and Chen, Y.
\newblock Beyond reverse kl: Generalizing direct preference optimization with
  diverse divergence constraints.
\newblock \emph{arXiv preprint arXiv:2309.16240}, 2023.

\bibitem[Weidinger et~al.(2021)Weidinger, Mellor, Rauh, Griffin, Uesato, Huang,
  Cheng, Glaese, Balle, Kasirzadeh, et~al.]{weidinger2021ethical}
Weidinger, L., Mellor, J., Rauh, M., Griffin, C., Uesato, J., Huang, P.-S.,
  Cheng, M., Glaese, M., Balle, B., Kasirzadeh, A., et~al.
\newblock Ethical and social risks of harm from language models.
\newblock \emph{arXiv preprint arXiv:2112.04359}, 2021.

\bibitem[Wiher et~al.(2022)Wiher, Meister, and Cotterell]{wiher2022decoding}
Wiher, G., Meister, C., and Cotterell, R.
\newblock On decoding strategies for neural text generators.
\newblock \emph{Transactions of the Association for Computational Linguistics},
  10:\penalty0 997--1012, 2022.

\bibitem[Workshop et~al.(2022)Workshop, Scao, Fan, Akiki, Pavlick, Ili{\'c},
  Hesslow, Castagn{\'e}, Luccioni, Yvon, et~al.]{workshop2022bloom}
Workshop, B., Scao, T.~L., Fan, A., Akiki, C., Pavlick, E., Ili{\'c}, S.,
  Hesslow, D., Castagn{\'e}, R., Luccioni, A.~S., Yvon, F., et~al.
\newblock Bloom: A 176b-parameter open-access multilingual language model.
\newblock \emph{arXiv preprint arXiv:2211.05100}, 2022.

\bibitem[Wu et~al.(2023)Wu, Hu, Shi, Dziri, Suhr, Ammanabrolu, Smith,
  Ostendorf, and Hajishirzi]{wu2023fine}
Wu, Z., Hu, Y., Shi, W., Dziri, N., Suhr, A., Ammanabrolu, P., Smith, N.~A.,
  Ostendorf, M., and Hajishirzi, H.
\newblock Fine-grained human feedback gives better rewards for language model
  training.
\newblock \emph{arXiv preprint arXiv:2306.01693}, 2023.

\bibitem[Yuan et~al.(2023)Yuan, Yuan, Tan, Wang, Huang, and
  Huang]{yuan2023rrhf}
Yuan, Z., Yuan, H., Tan, C., Wang, W., Huang, S., and Huang, F.
\newblock Rrhf: Rank responses to align language models with human feedback
  without tears.
\newblock \emph{arXiv preprint arXiv:2304.05302}, 2023.

\bibitem[Zheng et~al.(2023)Zheng, Chiang, Sheng, Zhuang, Wu, Zhuang, Lin, Li,
  Li, Xing, et~al.]{zheng2023judging}
Zheng, L., Chiang, W.-L., Sheng, Y., Zhuang, S., Wu, Z., Zhuang, Y., Lin, Z.,
  Li, Z., Li, D., Xing, E., et~al.
\newblock Judging llm-as-a-judge with mt-bench and chatbot arena.
\newblock \emph{arXiv preprint arXiv:2306.05685}, 2023.

\bibitem[Zhong et~al.(2024)Zhong, Feng, Xiong, Zhao, He, Bian, and
  Wang]{zhong2024dpo}
Zhong, H., Feng, G., Xiong, W., Zhao, L., He, D., Bian, J., and Wang, L.
\newblock Dpo meets ppo: Reinforced token optimization for rlhf.
\newblock \emph{arXiv preprint arXiv:2404.18922}, 2024.

\end{thebibliography}
\bibliographystyle{icml2024}

\newpage
\appendix
\onecolumn
\section{Mathematical Derivations}
\subsection{Proving the Relationship between Maximizing the Advantage Function and Enhancing the Expected Returns} \label{A_2}
\mylemmatwo*
\begin{proof}
Let trajectory $\tau:= (x, y^1, y^2, ...)$, and the notation $E_{\tau |\Tilde{\pi}}[\cdot]$ indicates that actions are
sampled from $\Tilde{\pi}$ to generate $\tau$. So we can get
\begin{align}
&\mathbb{E}_{x\sim \mathcal{D}}\left[V_{\Tilde{\pi}}([x])\right] - \mathbb{E}_{x\sim \mathcal{D}}\left[V_{\pi}([x])\right]\\
=&\mathbb{E}_{\tau|\Tilde{\pi}}\left[\sum_{t=1}^{\infty}\gamma^{t-1} R_{t}-V_{\pi}([x])\right]\\
=&\mathbb{E}_{\tau|\Tilde{\pi}}\left[\sum_{t=1}^{\infty}\gamma^{t-1} \left(R_{t}+\gamma V_{\pi}([x, y^{<t+1}])-V_{\pi}([x, y^{<t}])\right)\right]\\
=&\mathbb{E}_{\tau|\Tilde{\pi}}\left[\sum_{t=1}^{\infty}\gamma^{t-1} A_{\pi}([x, y^{<t}], y^t)\right]\\
=&\mathbb{E}_{\tau|\Tilde{\pi}}\left[\sum_{t=1}^{\infty}\gamma^{t-1} \mathbb{E}_{y^t \sim \Tilde{\pi}}\left[A_{\pi}([x, y^{<t}], y^t)\right]\right]\label{eq_v_A}
\end{align}
Since for any state $s_t=[{x}, y^{<t}]$, $\mathbb{E}_{z \sim \Tilde{\pi}}\left[A_{\pi}([{x}, y^{<t}], z)\right]\ge 0$, so we can obtain
\begin{align}
    \mathbb{E}_{x\sim 
    \mathcal{D}}\left[V_{\Tilde{\pi}}([x])\right] - \mathbb{E}_{x\sim 
    \mathcal{D}}\left[V_{\pi}([x])\right]
    \ge  0
\end{align}
\end{proof}
Our goal is to maximize the expected return of a parameterized policy $\pi_{\theta}$. According to Eq.\ref{eq_v_A}, what we need to do is $\max\limits_{\pi_{\theta}} \ \mathbb{E}_{{x}, {y^{<t}}\sim\mathcal{D},z\sim \pi_{\theta}(\cdot|[{x},y^{<t}])}\left[A_{\pi_{\mathrm{
ref}}}([{x},y^{<t}], z)\right]$. To prevent the excessive degradation of language models, we introduce a reverse KL divergence constraint, forming our objective function:
\begin{equation}
    \max_{\pi_{\theta}} \ \mathbb{E}_{{x}, {y}^{<t}\sim\mathcal{D},z\sim \pi_{\theta}(\cdot|[{x},y^{<t}])}\left[A_{\pi_{\mathrm{
ref}}}([{x},y^{<t}], z)-\beta D_{\mathrm{KL}}\left(\pi_{\theta}(\cdot|[{x},y^{<t}])||\pi_{\mathrm{ref}}(\cdot|[{x},y^{<t}])\right)\right]
\end{equation}


\subsection{Deriving the Mapping between the State-Action Function and the Optimal Policy}\label{A_4}
\begin{lemma}
    The constrained problem in Eq.~\ref{rpo_obj} has the closed-form solution:
\begin{equation}
\begin{aligned}
    \pi_{\theta}^*(z|[{x},y^{<t}])=
     \frac{\pi_{\mathrm{ref}}(z|[{x},y^{<t}])\exp\left(\frac{1}{\beta}Q_{\pi_{\mathrm{ref}}}([{x},y^{<t}],z)\right)}{Z([{x},y^{<t}];\beta)},
\end{aligned}
\end{equation}
where $Z([{x},y^{<t}];\beta) = \mathbb{E}_{z\sim \pi_{\mathrm{ref}}(\cdot|[{x},y^{<t}])}e^{\frac{1}{\beta}Q_{\pi_{\mathrm{ref}}}([{x},y^{<t}],z)}$
is the partition function.
\end{lemma}
\begin{proof}
\begin{align}
&\max_{\pi_{\theta}} \ \mathbb{E}_{z\sim \pi_{\theta}(\cdot|[{x},y^{<t}])}A_{\pi_{\mathrm{ref}}}([{x},y^{<t}], z)-\beta D_{\mathrm{KL}}\left(\pi_{\theta}(\cdot|[{x},y^{<t}])\|\pi_{\mathrm{ref}}(\cdot|[{x},y^{<t}])\right)\\
=&\max_{\pi_{\theta}}\ \mathbb{E}_{z\sim \pi_{\theta}(\cdot|[{x}, y^{<t}])}\left(\left(Q_{\pi_{\mathrm{ref}}}([{x},y^{<t}],z)-V_{\pi_{\mathrm{ref}}}([{x},y^{<t}])\right)+\beta\log \big(\frac{\pi_{\mathrm{ref}}(z|[{x}, y^{<t}])}{\pi_{\theta}(z|[{x}, y^{<t}])}\big)\right)\\
=&\max_{\pi_{\theta}}\ \beta\mathbb{E}_{z\sim \pi_{\theta}(\cdot|[{x}, y^{<t}])}\log\left(\frac{p(z|[{x},y^{<t}])e^{\frac{1}{\beta}Q_{\pi_{\mathrm{ref}}}([{x},y^{<t}],z)}}{\pi_{\theta}(z|[{x},y^{<t}])}\right)-V_{\pi_{\mathrm{ref}}}([{x}, y^{<t}])\\
=&\max_{\pi_{\theta}}\ \beta\mathbb{E}_{z\sim \pi_{\theta}(\cdot|[{x}, y^{<t}])}\log\left(\frac{\pi_{\mathrm{ref}}(z|[{x},y^{<t}])e^{\frac{1}{\beta}Q_{\pi_{\mathrm{ref}}}([{x},y^{<t}],z)}}{Z([{x}, y^{<t}];\beta)\pi_{\theta}(z|[{x},y^{<t}])}\right)-V_{\pi_{\mathrm{ref}}}([{x}, y^{<t}])+\beta\log Z([{x}, y^{<t}];\beta)\\
=&\max_{\pi_{\theta}}\  -\beta D_{\mathrm{KL}}\left(\pi_{\theta}(z|[{x}, y^{<t}])\bigg\|\frac{\pi_{\mathrm{ref}}(z|[{x},y^{<t}])e^{\frac{1}{\beta}Q_{\pi_{\mathrm{ref}}}([{x},y^{<t}],z)}}{Z([{x}, y^{<t}];\beta)}\right)-V_{\pi_{\mathrm{ref}}}([{x}, y^{<t}])+\beta\log Z([{x}, y^{<t}];\beta)
\end{align}    
where $Z([{x}, y^{<t}];\beta)$ is the partition function:
\begin{equation}
Z([{x},y^{<t}];\beta) = \mathbb{E}_{z\sim \pi_{\mathrm{ref}}(\cdot|[{x},y^{<t}])}\exp\left(\frac{1}{\beta}Q_{\pi_{\mathrm{ref}}}([{x},y^{<t}],z)\right)
\end{equation}
Based on the property of KL divergence, we can derive the relationship between the optimal policy and the state-action function:
\begin{equation}
    \pi_{\theta}^*(z|[{x},y^{<t}]) = \frac{\pi_{\mathrm{ref}}(z|[{x},y^{<t}])\exp\left(\frac{1}{\beta}Q_{\pi_{\mathrm{ref}}}([{x},y^{<t}],z)\right)}{Z([{x},y^{<t}];\beta)}
\end{equation}
\end{proof}

\subsection{Proving the Equivalence of the Bradley-Terry Model and the Regret Preference Model} \label{equivalence}
\begin{lemma}
Given a reward function $r(\mathrm{x}, \mathrm{y})$, assuming a relationship between token-wise rewards and the reward function represented by $r({x}, {y}) = \sum_{t=1}^T\gamma^{t-1}R([{x},y^{<t}], y^t)$, we can establish the equivalence between the Bradley-Terry model and the Regret Preference Model in the task of text generation alignment, i.e.,
\begin{equation}
    P_{\mathrm{BT}}({y}_1 \succ {y}_2 |{x})=\sigma\left(\sum_{t=1}^{T_1}\gamma^{t-1}A_{\pi}([{x},y_1^{<t}], y_1^{t}) - \sum_{t=1}^{T_2}\gamma^{t-1}A_{\pi}([{x},y_2^{<t}], y_2^{t})\right),
\end{equation}
where $\sigma(x)=1/(1+exp(-x))$ is the logistic sigmoid function.
\end{lemma}
\begin{proof}
According to the Bradley-Terry model, we have
\begin{equation}
    \begin{aligned}
        P_{\mathrm{BT}}({y}_1\succ {y}_2 | {x})=\frac{\exp(r({x}, {y}_1))}{\exp(r({x}, {y}_1))+\exp(r({x}, {y}_2))},
        \label{app_BT_model}
    \end{aligned}
\end{equation}
where $r({x}, {y})$ represents the overall reward of the pair $({x}, {y})$.

Based on assumption that $r({x}, {y}) = \sum_{t=1}^T\gamma^{t-1}R([{x},y^{<t}], y^t)$, we can get:
\begin{align}
    r({x}, {y}) &= \sum_{t=1}^{T}\gamma^{t-1}R([{x},y^{<t}], y^t)\\
    &=\sum_{t=1}^{T} \gamma^{t-1}(R([{x},y^{<t}], y^t) + \gamma V_{\pi}([{x},y^{<t+1}]) - \gamma V_{\pi}([{x},y^{<t+1}]))\\
    &=V_{\pi}([{x},y^{<1}]) + \sum_{t=1}^{T}\gamma^{t-1}\left( R([{x},y^{<t}], y^t) + \gamma V_{\pi}([{x},y^{<t+1}]) - V_{\pi}([{x},y^{<t}])\right) - \gamma^T V_{\pi}([{x},y^{<T+1}])\label{app_r}
\end{align}
Text generation is analogous to a deterministic contextual bandit, where the transition to the next state is certain given the current state and action, i.e., $p(s_{t+1}=[{x},y^{<t+1}]|s_{t}=[{x},y^{<t}], a_t=y^t)=1$, so we have:
\begin{align}
Q_{\pi}([{x},y^{<t}], y^t) &= R([{x},y^{<t}], y^t) + V_{\pi}([{x},y^{<t+1}])\label{app_Q}\\
A_{\pi}([{x},y^{<t}], y^t) &= Q_{\pi}([{x},y^{<t}], y^t) - V_{\pi}([{x},y^{<t}])
\label{app_A}
\end{align}

Next, note that  $y^T=\text{EOS}$ denotes the end of the text sequence. Therefore,
\begin{equation}
V_{\pi}([{x},y^{<T+1}])=\mathbb{E}_{\pi}\left[\sum_{k=0}^{\infty}\gamma^{k} R([{x},y^{<T+1+k}], y^{T+1+k})\bigg|s_{t}=[{x},y^{<T+1}]\right]=0\label{app_V}
\end{equation}

Substituting Eq.\ref{app_r} to Eq.\ref{app_V} into the Bradley-Terry model, we obtain

\begin{align}
    P_{\mathrm{BT}}({y}_1\succ {y}_2 | {x})&=\frac{\exp(r({x}, {y}_1))}{\exp(r({x}, {y}_1))+\exp(r({x}, {y}_2))}\nonumber\\
    =&\sigma\left(\left(V_{\pi}([{x},y_1^{<1}]) + \sum_{t=1}^{T_1}\left(\gamma^{t-1} A_{\pi}([{x},y_1^{<t}], y^t)\right)\right)-\left(V_{\pi}([{x},y_2^{<1}]) + \sum_{t=1}^{T_2}\left(\gamma^{t-1} A_{\pi}([{x},y_2^{<t}], y_2^t)\right)\right)\right)
\end{align}

Additionally, note that $y^{<1}=[\ ]$, so we can get
$$
V_{\pi}([{x},y_1^{<1}]) = V_{\pi}([{x},[\ ]]) = V_{\pi}([{x},y_2^{<1}])
$$
Therefore, 
\begin{align}
    P_{\mathrm{BT}}({y}_1\succ {y}_2 | {x}) &=\sigma\left(\left(V_{\pi}([{x},y_1^{<1}]) + \sum_{t=1}^{T_1}\left( \gamma^{t-1}A_{\pi}([{x},y_1^{<t}], y_1^t)\right)\right)-\left(V_{\pi}([{x},y_2^{<1}]) + \sum_{t=1}^{T_2}\left( \gamma^{t-1}A_{\pi}([{x},y_2^{<t}], y_2^t)\right)\right)\right)\\
    &=\sigma\left(\sum_{t=1}^{T_1}\left(\gamma^{t-1} A_{\pi}([{x},y_1^{<t}], y_1^t)\right)- \sum_{t=1}^{T_2}\left(\gamma^{t-1} A_{\pi}([{x},y_2^{<t}], y_2^t)\right)\right)
\end{align}
\end{proof}

\subsection{Deriving the \methodabb{} Objective Under the Bradley-Terry Model}\label{A_5}
\begin{theorem}
    In the KL-constrainted advantage function maximization problem corresponding to Eq.\ref{rpo_obj}, the Bradley-Terry model express the human preference probability in terms of the optimal policy $\pi_{\theta}^*$ and reference policy $\pi_{\mathrm{ref}}$:
    \begin{equation}
    \small
        P_{\mathrm{BT}}^*({y}_1 \succ {y}_2 |{x})=\sigma(u^*({x}, {y}_1, {y}_2) - \delta^*({x}, {y}_1, {y}_2)),
    \end{equation}
    where, $u({x}, {y}_1, {y}_2)$ refers to the difference in rewards implicitly defined by the language model $\pi_{\theta}$ and the reference model $\pi_{\mathrm{ref}}$ \cite{rafailov2023direct}, represented as
    \begin{equation}
    \small
    u({x}, {y}_1, {y}_2)=\beta\log\frac{\pi_{\theta}({y}_1\mid {x})}{\pi_{\mathrm{ref}}({y}_1\mid {x})}-\beta\log\frac{\pi_{\theta}({y}_2\mid {x})}{\pi_{\mathrm{ref}}({y}_2\mid {x})},
    \end{equation}
    and $\delta({x}, {y}_1, {y}_2)$ refers to the difference in sequential forward KL divergence between two pairs $({x}, {y}_1)$ and $({x}, {y}_2)$, weighted by $\beta$, expressed as
    \begin{equation}
    \delta({x}, {y}_1, {y}_2) =\beta D_{\mathrm{SeqKL}}\left({x},{y}_2;\pi_{\mathrm{ref}}\| \pi_{\theta}\right) - \beta D_{\mathrm{SeqKL}}\left({x},{y}_1;\pi_{\mathrm{ref}}\| \pi_{\theta}\right).
    \end{equation}
\end{theorem}
\begin{proof}
    According to the \cref{lemma4}, we have
    \begin{equation}
\begin{aligned}
    \pi_{\theta}^*(z|[{x},y^{<t}])
    =\frac{\pi_{\mathrm{ref}}(z|[{x},y^{<t}])\exp\left(\frac{1}{\beta}Q_{\pi_{\mathrm{ref}}}([{x},y^{<t}],z)\right)}{Z([{x},y^{<t}];\beta)}
\end{aligned}\label{eq_26}
\end{equation}
where $Z([{x},y^{<t}];\beta) = \mathbb{E}_{z\sim \pi_{\mathrm{ref}}(\cdot|[{x},y^{<t}])}e^{\frac{1}{\beta}Q_{\pi_{\mathrm{ref}}}([{x},y^{<t}],z)}$
is the partition function.

Rearrange Eq.\ref{eq_26}, we obtain
\begin{equation}
    \begin{aligned}
        Q_{\pi_{\mathrm{ref}}}([{x},y^{<t}],z)
        = \beta \log\frac{\pi_{\theta}^*(z|[{x}, y^{<t}])}{\pi_{\mathrm{ref}}(z|[{x}, y^{<t}])} + \beta \log Z([{x}, y^{<t}];\beta)
    \end{aligned}\label{eq_27}
\end{equation}

From \cref{lemma1}, We can get
\begin{equation}
    P_{\mathrm{BT}}({y}_1\succ {y}_2 | {x}) =\sigma\left(\sum_{t=1}^{T_1}\left(\gamma^{t-1} A_{\pi}([{x},y_1^{<t}], y_1^t)\right)- \sum_{t=1}^{T_2}\left(\gamma^{t-1} A_{\pi}([{x},y_2^{<t}], y_2^t)\right)\right)\label{eq_28}
\end{equation}

By leveraging Eq.\ref{eq_27}, we can derive
\begin{align}
    &\sum_{t=1}^{T}\gamma^{t-1}A_{\pi_{\mathrm{ref}}}([{x},y^{<t}], y^{t})\nonumber\\ 
        =&
    \sum_{t=1}^{T}\gamma^{t-1}\Big(Q_{\pi_{\mathrm{ref}}}([{x},y^{<t}],y^{t})-V_{\pi_{\mathrm{ref}}}([{x},y^{<t}])\Big)\\
    =&\sum_{t=1}^{T}\gamma^{t-1}\Big(Q_{\pi_{\mathrm{ref}}}([{x},y^{<t}],y^{t})-\mathbb{E}_{z\sim \pi_{\mathrm{ref}}}\left[Q_{\pi_{\mathrm{ref}}}([{x},y^{<t}], z)\right]\Big)\\
    =&\sum_{t=1}^{T}\gamma^{t-1}\left(\beta \log\frac{\pi_{\theta}^*(y^{t}|[{x}, y^{<t}])}{\pi_{\mathrm{ref}}(y^{t}|  ·[{x}, y^{<t}])} + \beta \log Z([{x}, y^{<t}];\beta)-\mathbb{E}_{z\sim \pi_{\mathrm{ref}}}\left[\beta \log\frac{\pi_{\theta}^*(z|[{x}, y^{<t}])}{\pi_{\mathrm{ref}}(z|[{x}, y^{<t}])} + \beta \log Z([{x}, y^{<t}];\beta) \right]\right)
\end{align}

Note that
\begin{equation}
    \mathbb{E}_{z\sim \pi_{\mathrm{ref}}}\left[\beta \log Z([{x}, y^{<t}];\beta)\right] = \beta \log Z([{x}, y^{<t}];\beta)
\end{equation}
Therefore,
\begin{align}
    \sum_{t=1}^{T}\gamma^{t-1}A_{\pi_{\mathrm{ref}}}([{x},y^{<t}], y^{t})&=\beta\sum_{t=1}^{T}\gamma^{t-1}\left( \log\frac{\pi_{\theta}^*(y^{t}|[{x}, y^{<t}])}{\pi_{\mathrm{ref}}(y^{t}|[{x}, y^{<t}])}-\mathbb{E}_{z\sim \pi_{\mathrm{ref}}}\left[\log\frac{\pi_{\theta}^*(z|[{x}, y^{<t}])}{\pi_{\mathrm{ref}}(z|[{x}, y^{<t}])}\right]\right)\\
    &=\beta\sum_{t=1}^{T}\gamma^{t-1}\left( \log\frac{\pi_{\theta}^*(y^{t}|[{x}, y^{<t}])}{\pi_{\mathrm{ref}}(y^{t}|[{x}, y^{<t}])}+D_{\mathrm{KL}}\left(\pi_{\mathrm{ref}}(\cdot|[{x},y^{<t}]) \|\pi_{\theta}^*(\cdot|[{x},y^{<t}])\right)\right)\\
&=\beta\sum_{t=1}^{T} \gamma^{t-1}\log\frac{\pi_{\theta}^*(y^{t}|[{x}, y^{<t}])}{\pi_{\mathrm{ref}}(y^{t}|[{x}, y^{<t}])}+\beta\sum_{t=1}^{T}\gamma^{t-1}D_{\mathrm{KL}}\left(\pi_{\mathrm{ref}}(\cdot|[{x},y^{<t}]) \|\pi_{\theta}^*(\cdot|[{x},y^{<t}])\right)
\end{align}
When substituting $\gamma=1$ into the expression, we obtain a more concise form:
\begin{align}
    \sum_{t=1}^{T}A_{\pi_{\mathrm{ref}}}([{x},y^{<t}], y^{t})&=\beta\sum_{t=1}^{T} \log\frac{\pi_{\theta}^*(y^{t}|[{x}, y^{<t}])}{\pi_{\mathrm{ref}}(y^{t}|[{x}, y^{<t}])}+\beta\sum_{t=1}^{T}D_{\mathrm{KL}}\left(\pi_{\mathrm{ref}}(\cdot|[{x},y^{<t}]) \|\pi_{\theta}^*(\cdot|[{x},y^{<t}])\right)\\
&=\beta\left(\log\frac{\pi_{\theta}^*({y}|{x})}{\pi_{\mathrm{ref}}({y}|{x})}+D_{\mathrm{SeqKL}}\left({x},{y};\pi_{\mathrm{ref}}\| \pi_{\theta}^*\right)\right)\label{eq_36}
\end{align}

We let
\begin{align}
    u({x}, {y}_1, {y}_2)&=\beta\log\frac{\pi_{\theta}({y}_1\mid {x})}{\pi_{\mathrm{ref}}({y}_1\mid {x})}-\beta\log\frac{\pi_{\theta}({y}_2\mid {x})}{\pi_{\mathrm{ref}}({y}_2\mid {x})},\label{app_u}\\
    \delta({x}, {y}_1, {y}_2) &=\beta D_{\mathrm{SeqKL}}\left({x},{y}_2;\pi_{\mathrm{ref}}\| \pi_{\theta}\right) - \beta D_{\mathrm{SeqKL}}\left({x},{y}_1;\pi_{\mathrm{ref}}\| \pi_{\theta}\right).\label{app_delta}
\end{align}
Substituting Eq.\ref{eq_36} to Eq.\ref{app_delta} into Eq.\ref{eq_28}, we arrive at:
\begin{equation}
        P_{\mathrm{BT}}^*({y}_1 \succ {y}_2 |{x})=\sigma(u^*({x}, {y}_1, {y}_2) - \delta^*({x}, {y}_1, {y}_2)).
    \end{equation}

\end{proof}

\section{\methodabb{} Implementation Details and Hyperparameters}\label{appendix_B}


PyTorch code for the \methodabb{} loss is provided below:
\begin{python}
import torch
import torch.nn.functional as F

def tdpo_loss(pi_logits, ref_logits, yw_idxs, yl_idxs, labels, beta, alpha, if_tdpo2):
    """
    pi_logits: policy logits. Shape: (batch_size, sequence_length, vocab_size), 
    ref_logits: reference logits. Shape: (batch_size, sequence_length, vocab_size)
    yw_idxs: preferred completion indices in [0,B-1], shape (T,)
    yl_idxs: dispreferred completion indices in [0,B-1], shape (T,)
    labels: labels for which to compute the log probabilities, Shape: (batch_size, sequence_length)
    beta: temperature controlling strength of KL penalty
    Each pair of (yw_idxs[i], yl_idxs[i]) represents the indices of a single preference pair.
    alpha: The weight factor adjusts the influence weight of kl divergence at each token
    if_tdpo2: Use method TDPO2 by default; if False, switch to TDPO1
    """

    pi_vocab_logps = pi_logits.log_softmax(-1)

    ref_vocab_ps = ref_logits.softmax(-1)
    ref_vocab_logps = ref_vocab_ps.log()
    
    pi_per_token_logps = torch.gather(pi_vocab_logps, dim=2, index=labels.unsqueeze(2)).squeeze(2)
    ref_per_token_logps = torch.gather(ref_vocab_logps, dim=2, index=labels.unsqueeze(2)).squeeze(2)

    per_position_rewards = pi_per_token_logps - ref_per_token_logps
    yw_rewards, yl_rewards = per_position_rewards[yw_idxs], per_position_rewards[yl_idxs]

    rewards = yw_rewards - yl_rewards
    
    # losses = -F.logsigmoid(beta * rewards)    # DPO loss function

    # =============================Difference with DPO=================================
    per_position_kl = (ref_vocab_ps * (ref_vocab_logps - pi_vocab_logps)).sum(-1)
    yw_kl, yl_kl = per_position_kl[yw_idxs], per_position_kl[yl_idxs]

    if not if_tdpo2:
        values = yw_rewards - yl_rewards - (yl_kl - yw_kl)
    else:
        values = yw_rewards - yl_rewards - alpha * (yl_kl - yw_kl.detach())    
        
    losses = -F.logsigmoid(beta * values)
    # =================================================================================

    return losses
\end{python}

Unless specified otherwise, we use a \(\alpha=0.5, \beta=0.1\), batch size of 64, and the RMSprop optimizer with a learning rate of 5e-6. We linearly warm up the learning rate from 0 to 5e-6 over 150 steps.

\clearpage
\section{Additional Experimental Results}\label{appendix_C}
\begin{figure*}[ht]
\vskip 0.2in
\centering
\subfigure[\label{IMDb_fkl:subfig1}]{\includegraphics[width=0.38\textwidth]{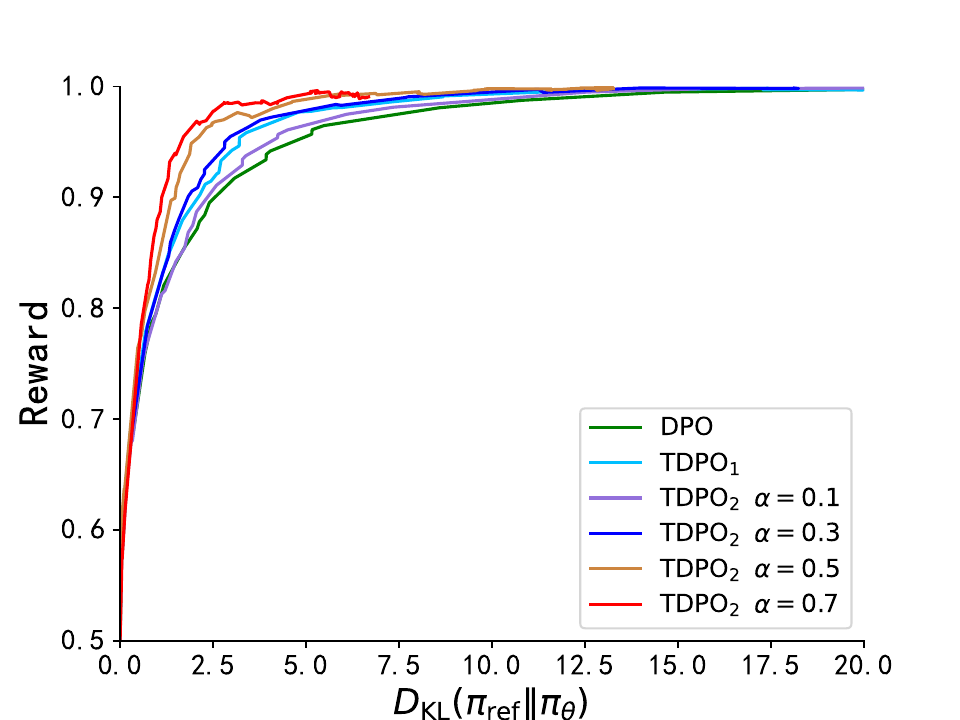}}
\subfigure[\label{IMDb_fkl:subfig2}]{\includegraphics[width=0.38\textwidth]{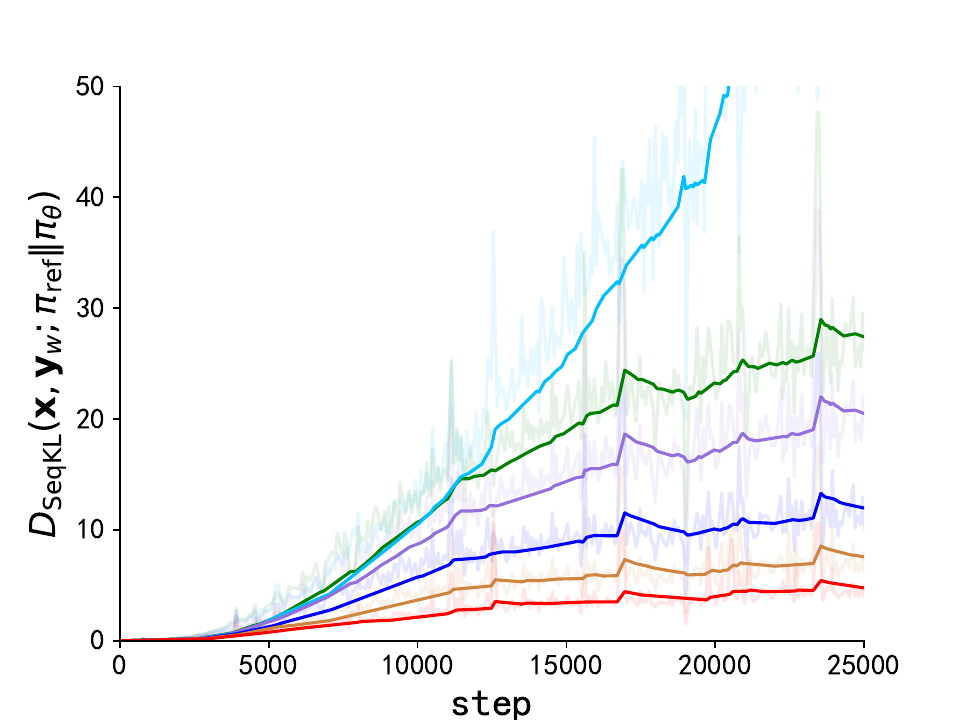}}
\subfigure[\label{IMDb_fkl:subfig3}]{\includegraphics[width=0.38\textwidth]{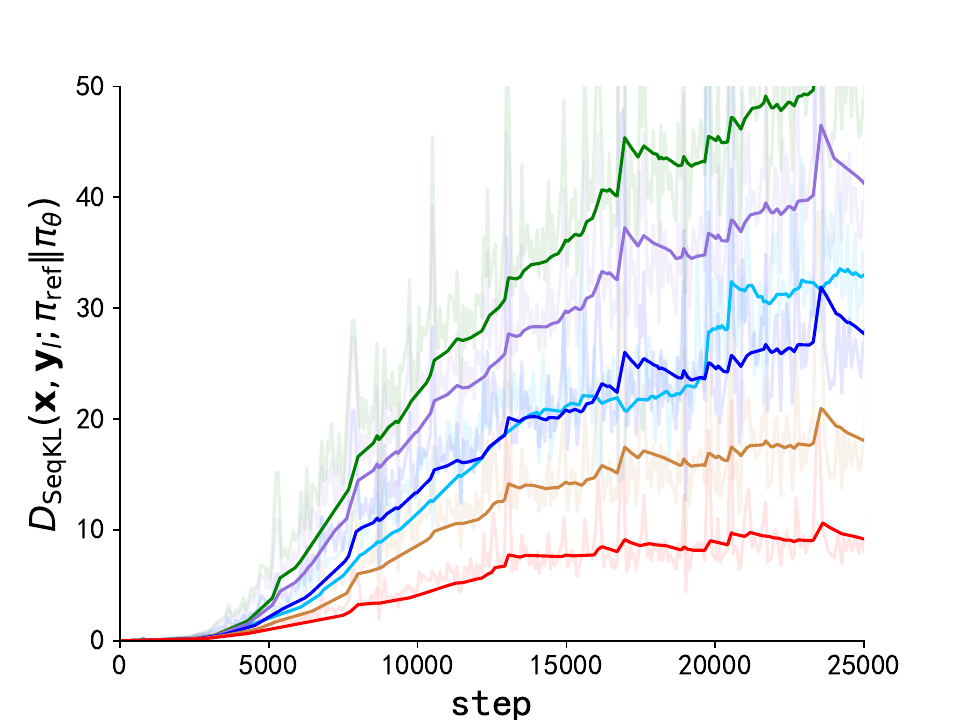}}
\subfigure[\label{IMDb_fkl:subfig4}]{\includegraphics[width=0.38\textwidth]{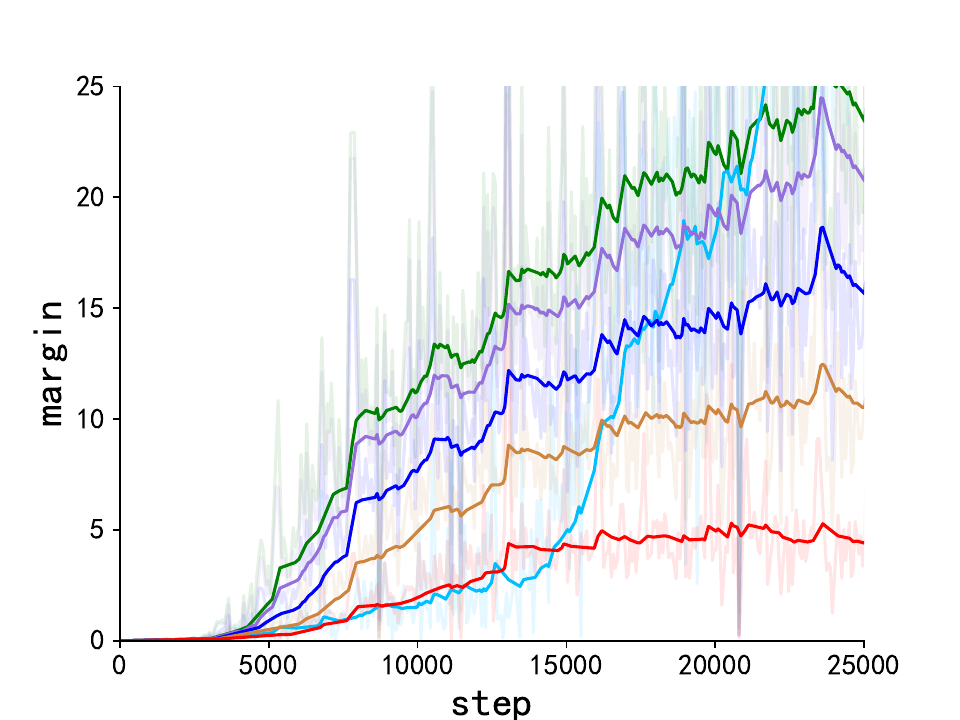}}
  \caption{The experiment on IMDb dataset. \cref{IMDb_fkl:subfig1} represents the frontier of expected reward and forward KL divergence with respect to the reference model.  \cref{IMDb_fkl:subfig2} and \cref{IMDb_fkl:subfig3} present the progression of sequential forward KL divergence on the preferred and dispreferred responses subset over training steps respectively. \cref{IMDb_fkl:subfig4} illustrates the difference between the sequential forward KL divergence on the dispreferred responses subset and that on the preferred responses subset throughout the training process.}
  \label{fig_fkl}
\vskip -0.2in
\end{figure*}

\end{document}